\documentclass{llncs}
\usepackage{amsmath}
\usepackage{amssymb}
\usepackage{times}
\usepackage{indentfirst}
\usepackage{graphicx}

\begin{document}

\title{Geometric lattice structure of covering-based rough sets through matroids}

\author{Aiping Huang, William Zhu~\thanks{Corresponding author.
E-mail: williamfengzhu@gmail.com(William Zhu)}}
\institute{ Lab of Granular Computing, Zhangzhou Normal University, Zhangzhou 363000, China}

\date{\today}

\date{\today}
\maketitle

\begin{abstract}
Covering-based rough set theory is a useful tool to deal with inexact, uncertain or vague knowledge in information systems.
Geometric lattice has widely used in diverse fields, especially search algorithm design which plays important role in covering reductions.
In this paper, we construct four geometric lattice structures of covering-based rough sets through matroids, and compare their relationships.
First, a geometric lattice structure of covering-based rough sets is established through the transversal matroid induced by the covering, and its
characteristics including atoms, modular elements and modular pairs are studied.
We also construct a one-to-one correspondence between this type of geometric lattices and transversal matroids in the context of covering-based rough sets.
Second, sufficient and necessary conditions for three types of covering upper approximation operators to be closure operators of matroids are presented.
We exhibit three types of matroids through closure axioms, and then obtain three geometric lattice structures of covering-based rough sets.
Third, these four geometric lattice structures are compared.
Some core concepts such as reducible elements in covering-based rough sets are investigated with geometric lattices.
In a word, this work points out an interesting view, namely geometric lattice, to study covering-based rough sets.

\end{abstract}

\section{Introduction}
Rough set theory~\cite{Pawlak82Rough} was proposed by Pawlak to deal with granularity in information systems.
It is based on equivalence relations.
However, the equivalence relation is rather strict, hence the applications of the classical rough set theory are quite limited.
For this reason, rough set theory has been extended to generalized rough set theory based on tolerance relation~\cite{SkowronStepaniuk96tolerance}, similarity relation~\cite{SlowinskiVanderpooten00AGeneralized} and arbitrary binary relation~\cite{LiuZhu08TheAlgebraic,Yao98OnGeneralizing,Yao98Relational,Yao98Constructive,ZhuWang06ANew}.
Through extending a partition to a covering, rough set theory is generalized to covering-based rough sets~\cite{QinGaoPei07OnCovering,WangZhuZhu10Structure,ZhuWang02Some,WangZhu12Quantitative}.
Because of its high efficiency in many complicated problems such as attribute reduction and rule learning in incomplete information/decision~\cite{QianLiangLiWangMa10Approximation}, covering-based rough set theory has been attracting increasing research interest~\cite{YaoYao12Covering,WangYangYangWu12Relationships}.

Lattice is suggested by the form of the Hasse diagram depicting it.
In mathematics, lattice are partially ordered sets in which any two elements have a unique supremum (also called a least upper bound or join) and a unique infimum (also called a greatest lower bound or meet). They encode the algebraic behavior of the entailment relation and such basic logical connectives as ``and" (conjunction) and ``or"(disjunction), which result in adequate algebraic semantics for a variety of logical systems.
Lattices, especially geometric lattices, are one of the most important algebraic structures and are used extensively in both theoretical and
applicable fields, such as data analysis, formal concept analysis~\cite{WangLiuCao10ANew,Wille82Restructuring,YaoChen04Rough} and domain theory~\cite{Birhoff95Lattice}.

Matroid theory~\cite{Oxley93Matroid,Lai01Matroid} borrows extensively from linear algebra and graph theory.
There are dozens of equivalent ways to define a matroid.
Significant definitions of matroid include those in terms of independent sets, bases, circuits, closed sets or flats, closure operators, and rank functions, which provides well-established platforms to connect with other theories.
In application, matroids have been widely used in many fields such as combinatorial optimization, network flows, and algorithm design, especially greedy algorithm design~\cite{Lawler01Combinatorialoptimization,Edmonds71Matroids}.
Some works on the connection between rough sets and matroids have been conducted~\cite{WangZhu11Matroidal,WangZhuMin11Transversal,WangZhuZhuMin12Matroidal,ZhuWang11Matroidal}.

In this paper, we pay our attention to geometric lattice structures of covering based-rough sets through matroids.
First, a geometric lattice in covering-based rough sets is generated by the transversal matroid induced by a covering.
Moreover, we study the characteristics of the geometric lattice, such as atoms, modular elements and modular pairs.
We also point out a one-to-one correspondence between this type of geometric lattices and transversal matroids in the context of covering-based rough sets.
Second, generally, covering upper approximation operators are not necessarily closure operators of matroids.
Then we present sufficient and necessary conditions for three types of covering upper approximation operators to be closure operators of matroids, and exhibit representations of corresponding special coverings.
We study the properties of these matroids, and their closed-set lattices which are also geometric lattices.
Third, we compare these four geometric lattices through corresponding matroids.
Furthermore, some core concepts such as reducible and immured elements in covering-based rough sets are studied by geometric lattices.

The rest of this paper is organized as follows.
In Section~\ref{Preliminaries}, we recall some fundamental concepts related to covering-based rough sets, lattices and matroids.
Section~\ref{S:Thepropertiesoftheclosedsetlatticeofmatroidinducedbycovering} establishes a geometric lattice structure of covering-based rough sets through the transversal matroid induced by a covering.
In Section~\ref{S:Matroidalstructuresbasedonthreetypesofupperapproximations}, we present three geometric lattice structures of covering-based rough sets through three types of approximation operators.
Section~\ref{S:Therelationshipamongfourtypesofmatroidalstructuresandclosedsetlatticestructures} studies the relationship among these four geometric lattice structures.
This paper is concluded and further work is pointed out in Section~\ref{S:Conclusions}.

\section{Preliminaries}
\label{Preliminaries}

In this section, we review some basic concepts of matroids, lattices and covering-based rough sets.

\subsection{Matriod}

\begin{definition}(Matroid)~\cite{Oxley93Matroid}
A matroid is an ordered pair $(E,\mathcal{I})$ consisting of a finite set $E$ and a collection $\mathcal{I}$ of subsets of $E$ satisfying the following three conditions:\\
(1) $\emptyset \in \mathcal{I}$;\\
(2) If $I\in \mathcal{I}$ and $I^{'}\subseteq I$, then $I^{'}\in \mathcal{I}$;\\
(3) If $I_{1},I_{2}\in \mathcal{I}$ and $|I_{1}|<|I_{2}|$, then there is an element $e\in I_{2}-I_{1}$ such that $I_{1}\bigcup e\in \mathcal{I}$, where $|X|$ denotes the cardinality of $X$.
\end{definition}

Let $M(E,\mathcal{I})$ be a matroid.
The members of $\mathcal{I}$ are the independent sets of $M$.
A set in $\mathcal{I}$ is maximal, in the sense of inclusion, is called a base of the matroid $M$.
If $A\notin \mathcal{I}$, $A$ is called dependent set.
In the sense of inclusion, a minimal dependent subset of $E$ is called a circuit of the matroid $M$.
If $\{a\}$ is a circuit, we call $\{a\}$ a loop.
Moreover, if $\{a,b\}$ is a circuit, then $a$ and $b$ are said to be parallel.
A matroid is called simple matroid if it has no loops and no parallel elements.
The rank function of a matroid is a function $r_{M}:2^{E}\rightarrow N$ defined by $r_{M}(X)=max\{|I|: I\subseteq X, I\in \mathcal{I}\}$ $(X\subseteq E)$.
For each $X\subseteq E$, we say $cl_{M}(X)=\{a\in E:r_{M}(X)=r_{M}(X\bigcup \{a\})\}$ is the closure of $X$ in $(E,\mathcal{I})$.
When there is no confusion, we use the symbol $cl(X)$ for short.
$X$ is called a closure set if $cl(X)=X$.

The rank function of a matriod, directly analogous to a similar theorem of linear algebra, has the following proposition.

\begin{proposition}\cite{Oxley93Matroid}
Let $M(E,\mathcal{I})$ be a matroid and $r_{M}$ is rank function of $M$. For all $X,Y\subseteq E$, the following properties hold:\\
(R1) For all $X\in 2^{E}$, $0\leq r_{M}(X)\leq|X|$;\\
(R2) If $X\subseteq Y\subseteq E$, then $r_{M}(X)\leq r_{M}(Y)$;\\
(R3) If $X,Y\subseteq E$, then $r_{M}(X\bigcup Y)+r_{M}(X\bigcap Y)\leq r_{M}(X)+r_{M}(Y)$.
\end{proposition}

The following proposition is the closure axiom of a matroid.
It means that a operator satisfies (1)-(4) if and only if the operator is the closure operator of a matroid.

\begin{proposition}\cite{Oxley93Matroid}
Let $E$ be a set. A function $cl_{M}:2^{E}\rightarrow 2^{E}$ is the closure operator of a matroid on $E$ if and only if it satisfies the following conditions:\\
(1) If $X\subseteq E$, then $X\subseteq cl_{M}(X)$.\\
(2) If $X\subseteq Y\subseteq E$, then $cl_{M}(X)\subseteq cl_{M}(Y)$.\\
(3) If $X\subseteq E$, $cl_{M}(cl_{M}(X))=cl_{M}(X)$.\\
(4) If $X\subseteq E,x\in E$, and $y\in cl_{M}(X\bigcup \{x\})-cl_{M}(X)$, then $x\in cl_{M}(X\bigcup \{y\})$.
\end{proposition}

Transversal theory is a branch of a matroid theory. It shows how to induce a matroid, namely, transversal matroid, by a family of subsets of a set. Hence, the transversal matroid establishes a bridge between collections of subsets of a set and matroids.

\begin{definition}(transversal)\cite{Oxley93Matroid}
Let $S$ be a nonempty finite set, $J=\{1,2,\cdots,m\}$. $\mathcal{F}$ denotes the family $\{F_{1},F_{2},\cdots,F_{m}\}$ of subsets of $S$. A transversal or system of distinct representatives of $\{F_{1},F_{2},\cdots,F_{m}\}$ is a subset $\{e_{1},e_{2},\cdots,e_{m}\}$ of $S$ such that $e_{i}\in F_{i}$ for all $i$ in $J$. If for some subset $K$ of $J$, $X$ is a transversal of $\{F_{i}:i\in K\}$, then $X$ is said to be a partial transversal of $\{F_{1},F_{2},\cdots,F_{m}\}$.
\end{definition}

\begin{example}
Let $S=\{1,2,3,4\}$, $F_{1}=\{2,3\},F_{2}=\{4\},F_{3}=\{2,4\}$. For $\mathcal{F}=\{F_{1},F_{2},F_{3}\}$, $T=\{2,3,4\}$ is a transversal of $\mathcal{F}$, since $2\in F_{3}$, $3\in F_{1}$, $4\in F_{2}$. $T^{'}=\{2,4\}$ is a partial transversal of $\mathcal{F}$, since there exists $K=\{1,2\}\subseteq J$ such that $T^{'}$ is a transversal of $\{F_{i}:i\in K\}$.
\end{example}

\begin{proposition}\cite{Oxley93Matroid}
Let $\mathcal{F}=\{F_{i}:i\in J\}$ be a family of subsets of $E$. $M(\mathcal{F})=(E, \mathcal{I}(\mathcal{F}))$ is a matroid where $\mathcal{I}(\mathcal{F})$ is the family of all partial transversals of $\mathcal{F}$.
\end{proposition}

\begin{definition}\cite{WangZhuMin11Transversal}
Let $\mathcal{F}=\{F_{i}:i\in J\}$ be a family of subsets of $E$. We say $M(\mathcal{F})=(E,\mathcal{I}(\mathcal{F}))$ is the transversal matroid induced by $\mathcal{F}$.
\end{definition}

\subsection{Lattice}

Let $P$ be an ordered set and $a,b \in P$.
We say that $a$ is covered by $b$ (or $b$ covers $a$) if $a<b$ and there is no element $c$ in $P$ with $a<c<b$.
A chain in $P$ from $x_{0}$ to $x_{n}$ is a subset $\{x_{0},x_{1},\cdots,x_{n}\}$ of $P$ such that $x_{0}<x_{1}<\cdots<x_{n}$.
The length of such a chain is $n$, and the chain is maximal if $x_{i}$ covers $x_{i-1}$ for all $i\in \{1,2,\cdots,n\}$.
A poset $(\mathcal{L},\leq)$ is a lattice if $a\bigvee b$ and $a\bigwedge b$ exist for all $a,b\in \mathcal{L}$. If $\mathcal{L}$ is a lattice with zero element $0$, then $a\in \mathcal{L}$ is called an atoms of $\mathcal{L}$ if $a$ covers $0$. If, for every pair $\{a,b\}$ of elements of $P$ with $a<b$, all maximal chains from $a$ to $b$ have the same length, then $P$ is said to satisfy the Jordan-Dedekind chain condition. The height $h_{P}(y)$ of an element $y$ of $P$ is the maximum length of a chain from $0$ to $y$. Thus, in particular, the atoms of $P$ are precisely the elements of height one. It is not difficult to check that every finite lattice has a zero and the one. A finite lattice $\mathcal{L}$ is called semimodular if it satisfies the Jordan-Dedekind chain condition and for every pair $x,y$ of elements of $\mathcal{L}$, the equality $h_{\mathcal{L}}(x)+h_{\mathcal{L}}(y)\geq h_{\mathcal{L}}(x\bigvee y)+h_{\mathcal{L}}(x\bigwedge y)$ holds. A geometric lattice is a finite semimodular lattice in which every element is join of atoms.

\begin{definition}\cite{Birhoff95Lattice}
Let $\mathcal{L}$ be a lattice, $a,b\in \mathcal{L}$.\\
(ME) For all $x, z\in \mathcal{L}$, $x\geq z$ implies $x\bigwedge(a\bigvee z)=(x\bigwedge a)\bigvee z$, $a$ is called a modular element of $\mathcal{L}$.\\
(MP) For all $z\in \mathcal{L}$, $b\geq z$ implies $b\bigwedge(a\bigvee z)=(b\bigwedge a)\bigvee z$, $(a,b)$ is called a modular pair of $\mathcal{L}$.
\end{definition}

\begin{lemma}\cite{Birhoff95Lattice}
Let $\mathcal{L}$ be a semimodular lattice, then $(a,b)$ is a modular pair if and only if $h_{\mathcal{L}}(a \bigvee b)+h_{\mathcal{L}}(a \bigwedge b)=h_{\mathcal{L}}(a)+h_{\mathcal{L}}(b)$ for all $a,b \in \mathcal{L}$
\end{lemma}
\subsection{Closed-set lattice of matroid}

If $M$ is a matroid, then $\mathcal{L}(M)$ will denote the set of all closed sets of $M$ ordered by inclusion. For a matroid $M$, the zero of $\mathcal{L}(M)$ is $cl_{M}(\emptyset)$, while the one is $E$.

\begin{lemma}\cite{Oxley93Matroid}
$\mathcal{L}(M)$ is a lattice, for all closed sets $X$ and $Y$ of $M$,
$X\bigwedge Y=X\bigcap Y$ and $X\bigvee Y=cl_{M}(X\bigcup Y)$.
\end{lemma}

\begin{lemma}\cite{Oxley93Matroid}
Let $h_{\mathcal{L}(M)}$ be the height function of lattice $\mathcal{L}(M)$, $r_{M}$ the rank function of $M$. For all $X\in \mathcal{L}(M)$, $h_{\mathcal{L}(M)}(X)=r_{M}(X)$.
\end{lemma}

\begin{lemma}\cite{Oxley93Matroid}
A lattice $\mathcal{L}$ is geometric if and only if it is the lattice of closure sets of a matroid.
\end{lemma}

Lemma 4 give another definition of a geometric lattice.
In face, it is the closed-set lattice of a matroid.

\subsection{Covering-based rough sets}

\begin{definition}(Covering and partition)
Let $E$ be a universe of discourse, $\mathcal{C}$ a family of subsets of $E$ and none of subsets in $\mathcal{C}$ be empty. If $\bigcup \mathcal{C}=E$, $\mathcal{C}$ is called a covering of $E$. The member of $\mathcal{C}$ is called a covering block. If $\mathcal{P}$ is a covering of $E$ and it is a family of pairwisely disjoint subsets of $E$, $P$ is called a partition of $E$.
\end{definition}

It is clear that a partition of $E$ is certainly a covering of $E$, so the concept of a covering is an extension of the concept of a partition.

Let $E$ be a finite set and $R$ be an equivalent relation on $E$. $R$ will generate a partition $E/R=\{Y_{1}, Y_{2},\cdots, Y_{m}\}$ on $E$, where $Y_{1},Y_{2},\cdots Y_{m}$ are the equivalence classes generated by $R$. $\forall X\subseteq U$, the lower and upper approximations of $X$, are, respectively, defined as follows:\\
$R_{\ast}(X)=\bigcup\{Y_{i}\in U/R: Y_{i}\subseteq X\}$\\
$R^{\ast}(X)=\bigcup \{Y_{i}\in U/R: Y_{i}\bigcap X\neq \emptyset\}$.

\begin{definition}(Indiscernible neighborhood and neighborhood)~\cite{ZhuWang07OnThree,Zhu07Topological}
Let $\langle E, \mathcal{C}\rangle$ be a covering approximation space, $x\in U$.\\
$\bigcup\{K:x\in K\in \mathcal{C}\}$ is called the indiscernible neighborhood of $x$ and denoted as $I(x)$.\\
$\bigcap \{K:x\in K\in \mathcal{C}\}$ is called the neighborhood of $x$ and denoted as $N(x)$.
\end{definition}

\begin{definition}~\cite{ZhuWang03Reduction}
Let $\mathcal{C}$ be a covering of a domain $E$ and $K\in \mathcal{C}$. If $K$ is a union of some sets in $\mathcal{C}-\{K\}$, we say $K$ is reducible in $\mathcal{C}$; otherwise $K$ is irreducible. If every element in $\mathcal{C}$ is irreducible, we say $\mathcal{C}$ is irreducible; otherwise $\mathcal{C}$ is reducible.
\end{definition}

\begin{definition}~\cite{ZhuWang03Reduction}
For a covering $\mathcal{C}$ of a universe $E$, the new irreducible covering through the above reduction is called the reduct of $\mathcal{C}$ and denoted by $reduct(\mathcal{C})$.
\end{definition}

\begin{definition}~\cite{ZhuWang07OnThree}
Let $\mathcal{C}$ be a covering of $E$ and $K$ an element of $\mathcal{C}$. If there exists another element $K^{'}$ of $\mathcal{C}$ such that $K\subset K^{'}$, we say that $K$ is an immured element of covering $\mathcal{C}$.
\end{definition}

\begin{definition}~\cite{ZhuWang07OnThree}
Let $\mathcal{C}$ be a covering of $E$. When we remove all immured elements from $\mathcal{C}$, the set of all remaining elements is still a covering of $E$, and this new covering has no immured element. We called this new covering an exclusion of $\mathcal{C}$, and it is denoted by $exclusion(\mathcal{C})$.
\end{definition}

Zakowski first extended Pawlak's rough set theory from partition to covering~\cite{Zakowski83Approximations}. Pomykala studied the second type of covering rough set model~\cite{Pomykala87Approximation}. While the sixth type of covering-based upper approximation was first defined in~\cite{XuWang05On}. Qin et al. first defined the seventh type of covering-based upper approximation in \cite{QinGaoPei07OnCovering}.

\begin{definition}
Let $\mathcal{C}$ be a covering of $E$.
The covering upper approximation operators $SH,XH,VH:P(E)\rightarrow P(E)$ are defined as follows: $\forall X\in P(E)$,\\
$SH(X)=\bigcup \{K\in \mathcal{C}:K\bigcap X\neq \emptyset\}=\bigcup\{I(x):x\in X\}$.\\
$XH(X)=\{x:N(x)\bigcap X \neq \emptyset\}$.\\
$VH(X)=\bigcup\{N(x):N(x)\bigcap X \neq \emptyset\}$\\
$SH_{\mathcal{C}},XH_{\mathcal{C}},VH_{\mathcal{C}}$ are called the second, the sixth, the seventh covering upper approximation operators with respect to the covering $\mathcal{C}$, respectively. When there is no confusion, we omit $\mathcal{C}$ at the lowercase.
\end{definition}


\section{A geometric lattice structure of covering-based rough sets through transversal matroid}
\label{S:Thepropertiesoftheclosedsetlatticeofmatroidinducedbycovering}

As we know, if $M$ is a matroid and $\mathcal{L}(M)$ denotes the set of all closed set of $M$ ordered by inclusion, then $\mathcal{L}(M)$ is a geometric lattice.
In this section, we study the properties such as atoms, modular elements and modular pairs of this type of geometric lattice through transversal matroid induced by a covering.
We also study the structure of matroid induced by this type of geometric lattice.
It is interesting to find that there is a one-to-one correspondence between this type of geometric lattices and transversal matroids in the context of covering-based rough sets.

Let $E$ be a nonempty finite set and $\mathcal{C}$ a covering of $E$.
As shown in Definition 3, $M(\mathcal{C})=(E, \mathcal{I}(\mathcal{C}))$ is the transversal matroid induced by covering $\mathcal{C}$. $\mathcal{L}(M(\mathcal{C}))$ is the set of all closed sets of $M(\mathcal{C})$.
Especially, $\mathcal{L}(M(\mathcal{P}))$ is the set of all closed sets of the transversal matroid induced by partition $\mathcal{P}$.
Based on Lemma 4, we know $\mathcal{L}(M(\mathcal{C}))$ and $\mathcal{L}(M(\mathcal{P}))$ are geometric lattice.

The theorem below connects a covering with the closure of $\emptyset$.
In fact, $\emptyset\in \mathcal{L}(M(\mathcal{F}))$ if and only if $\mathcal{F}$ is a covering.

\begin{theorem}
Let $M(\mathcal{F})$ be a transversal matriod induced by $\mathcal{F}=\{F_{1},F_{2},\cdots,F_{m}\}$, $\forall 1\leq i\leq m$, $F_{i}\neq \emptyset$.
$cl_{\mathcal{F}}(\emptyset)=\emptyset$ if and only if $\mathcal{F}$ is a covering.
\end{theorem}

\begin{proof}
"$\Leftarrow$": According to the definition of transversal matroid, any partial transversal is an independent set of transversal matroid.
Since $\mathcal{F}$ is a covering, any single-point set is an independent set.
Based on the definition of closure operator of a matriod, we have $cl_{M(\mathcal{F})}(\emptyset)=\emptyset$.

"$\Rightarrow$": Since $cl_{M(\mathcal{F})}(\emptyset)=\emptyset$, any single-point set is an independent set, that is, for all $x\in E$, there exists $1 \leq i_{x}\leq m$ such that $x\in F_{i_{x}}\subseteq E$.
Hence, $E=\bigcup_{x\in E}\{x\}\subseteq \bigcup_{x\in E}F_{i_{x}}\subseteq \bigcup_{i=1}^{m}F_{i}\subseteq E$.
Thus $\bigcup_{i=1}^{m}F_{i}=E$.
For all$ 1\leq i \leq m$, $F_{i}\neq \emptyset$ and $\bigcup_{i=1}^{m}F_{i}=E$, hence $\mathcal{F}$ is a covering.
\end{proof}

\begin{lemma}
Let $\mathcal{C}$ be a covering of $E$. For all $x\in E$, $cl_{M(\mathcal{C})}(\{x\})$ is an atom of $\mathcal{L}(M(\mathcal{C}))$.
\end{lemma}

\begin{proof}
Since $\mathcal{C}$ is a covering and the definition of transversal matroid, we konw any single-point set is an independent set.
Thus $\forall x\in E$, $r_{M(\mathcal{C})}(cl_{M(\mathcal{C})}(\{x\}))=r_{M(\mathcal{C})}(\{x\})=1$. Hence, $\{cl_{M(\mathcal{C})}(x):x\in E\}$ is the set of atoms of lattice $\mathcal{L}(M(\mathcal{C}))$.
\end{proof}

Lemma 5 does not establish the concrete form of $cl_{M(\mathcal{C})}(\{x\})$.
In order to solve that problem, we define two sets as follows.

\begin{definition}
Let $\mathcal{C}=\{K_{1}, K_{2},\cdots, K_{m}\}$ be a covering of a finite set $E=\{x_{1}, x_{2},\\ \cdots, x_{n}\}$.
We define\\
(i) $A=\{K_{i}-\bigcup_{j=1,j\neq i }^{m}K_{j}:K_{i}-\bigcup_{j=1,j\neq i }^{m}K_{j} \neq \emptyset, i\in \{1, 2, \cdots, m\}\}\\
~~~~~~~~~=\{A_{1},A_{2},\cdots, A_{s}\}$.\\
(ii) $B=E-\bigcup_{i=1}^{s} A_{i}$.
\end{definition}

\begin{remark}
For all $i\in \{1, 2, \cdots, s\}$, $\forall x\in A_{i}$, there exists only one block such that $x$ belongs to it, and there exist at least two blocks such that $y$ belongs to them for all $y\in B$.
\end{remark}

The following two propositions establish the characteristics of $A$ and $B$.

\begin{proposition}
Let $\mathcal{C}$ be a covering of $E$. $\{A_{1},A_{2},\cdots,A_{s}\}\bigcup \{\{x\}:{x\in B}\}$ forms a partition of $E$.
\end{proposition}

\begin{proof}
Let $P=A\bigcup \{\{x\}:x\in B\}=\{A_{1},A_{2},\cdots,A_{s}\}\bigcup \{\{x\}:{x\in B}\}$.
According to Definition 12, we know $\bigcup_{i=1}^{s}A_{i}\bigcup \{\{x\}:x\in B\}=E$.
Now we need to prove $\forall P_{1}, P_{2}\in P$, $P_{1}\bigcap P_{2}=\emptyset$.
According to the definition of $A$, if $P_{1},P_{2}\in A$, then $P_{1}\bigcap P_{2}=\emptyset$.
If $P_{1},P_{2}\in \{\{x\}:x\in B\}$, then $P_{1}\bigcap P_{2}=\emptyset$ because $P_{i}$ and $P_{j}$ are single-points.
If $P_{1}\in A, P_{2}\in  \{\{x\}:x\in B\}$, then $P_{1}\bigcap P_{2}=\emptyset$ because $B\bigcap \bigcup_{k=1}^{s}A_{k}=\emptyset$ and $A_{i}\subseteq \bigcup_{k=1}^{s}A_{k}$ and $A_{j}\subseteq B$.
\end{proof}

\begin{proposition}
$\mathcal{C}$ is a partition if and only if $B=\emptyset$.
\end{proposition}

\begin{proof}
According to the definition of $A$ and $B$, the necessity is obvious.
Now we prove the sufficiency.
If $\mathcal{C}$ is not a partition, then there exist $K_{i},K_{j}\in \mathcal{C}$ such that $K_{i}\bigcap K_{j}\neq \emptyset$.
Thus there exists $x\in E$ such that $x\in K_{i}\bigcap K_{j}$, that is, there exist at least $K_{i},K_{j}\in\mathcal{C}$ such that $x$ belongs to them, hence $x\in B$.
That contradicts the assumption that $B=\emptyset$.
\end{proof}

The following theorem shows the concrete form of atoms of lattice $\mathcal{L}(M(\mathcal{C}))$.

\begin{theorem}
Let $\mathcal{C}$ be a covering of $E$.
$\{A_{1},A_{2},\cdots,A_{s}\}\bigcup \{\{x\}:{x\in B}\}$ is the set of atoms of lattice $\mathcal{L}(M(\mathcal{C}))$.
\end{theorem}

\begin{proof}
According to the definition of $A_{i}$, we may as well suppose $A_{i}=K_{h}-\bigcup_{j=1,j\neq h}^{m}\\K_{j}$.
Based on $\mathcal{C}$ is a covering and the definition of transversal matroid, we know any single-point set is an independent set, thus $\forall x\in A_{i}$, $\{x\}$ is an independent set.
$\forall y\in A_{i}$ and $y\neq x$, we know $x,y\in K_{h}$ and $x,y\notin K_{j}$ for all $1 \leq j\leq m,~j\neq h$, thus $x$ and $y$ cannot be chosen from different blocks in the covering $\mathcal{C}$.
That shows that $\{x,y\}$ is not an independent set according to the definition of transversal matroid.
Hence, $\{x\}$ is a maximal independent set included in $A_{i}$, that is, $r_{M(\mathcal{C})}(A_{i})=1$.
Next, we need to prove $A_{i}$ is a closed set.
Since $A_{i}\subseteq cl_{M(\mathcal{C})}(A_{i})$, we need to prove $cl_{M(\mathcal{C})}(A_{i})\subseteq A_{i}$, that is, $x\notin A_{i}$ implies $x\notin cl_{M(\mathcal{C})}(A_{i})$.
If $y\notin A_{i}$, based on the fact that $\mathcal{C}$ is a covering and the definition of $A_{i}$, then there exists $j\neq h$ such that $y\in K_{j}$.
Thus $\{x,y\}$ is an independent set.
That implies $y\notin cl_{M(\mathcal{C})}(A_{i})$, thus $cl_{M(\mathcal{C})}(A_{i})=A_{i}$.
Hence, $A_{i}\in \mathcal{L}(M(\mathcal{C}))$.
Combining Lemma 3 with $r_{M(\mathcal{C})}(A_{i})=1$, we know for all $1 \leq i \leq m$, $A_{i}$ is an atom of lattice $\mathcal{L}(M(\mathcal{C}))$.

According to the definition of transversal matroid and the fact that $\mathcal{C}$ is a covering, any single-point set is an independent set.
Thus for all $x\in B$, $r_{M(\mathcal{C})}(\{x\})=1$.
$\forall y\in E$ and $y\neq x$, if $y\in B$, then there exist at least two blocks containing $y$ according to the definition of $B$.
We may as well suppose $y\in K_{k},K_{t}$ and $x\in K_{l},K_{p}$, where $\{K_{k}, K_{t}\}$ may be the same as $\{K_{l}, K_{p}\}$.
Based on this, $\{x,y\}$ is an independent set.
This implies $y\notin cl_{M(\mathcal{C})}(\{x\})$.
If $y\notin B$, then we may as well suppose $y\in A_{i}$, thus $y\in K_{h}$ for the definition of $A_{i}$, where $K_{h}$ may be the same with $K_{l}$ or $K_{p}$.
Based on this, $x$ and $y$ can be chosen from different blocks in covering $\mathcal{C}$, thus $\{x,y\}$ is an independent set. That implies  $y\notin cl_{M(\mathcal{C})}(\{x\})$.
From above discussion, we have $cl_{M(\mathcal{C})}(\{x\})=\{x\}$.
Hence, $\{x\}\in \mathcal{L}(M(\mathcal{C}))$ for all $x\in B$.
Combining Lemma 3 with $r_{M(\mathcal{C})}(\{x\})=1$, we know $\{x\}$ is an atom of lattice $\mathcal{L}(M(\mathcal{C}))$ for all $x\in B$.

Next, we will prove the set of atoms of lattice $\mathcal{L}(M(\mathcal{C}))$ can not be anything but $\{A_{1},A_{2},\cdots,A_{s}\}\bigcup \{\{x\}:{x\in B}\}$.
According to Lemma 5, we know $\{cl_{M(\mathcal{C})}(\{x\}):x\in E\}$ is the set of atoms of lattice $\mathcal{L}(M(\mathcal{C}))$.
Similar to the proof of the second part, we know that if $x\in B$ then $cl_{M(\mathcal{C})}(\{x\})=\{x\}$.
If $x\notin B$, then $x$ belongs to one of elements in $A$.
We may as well suppose $x\in A_{i}$.
Combining $A_{i}$ is an atom with $\emptyset \subseteq cl_{M(\mathcal{C})}(\{x\})\subseteq cl_{M(\mathcal{C})}(A_{i})=A_{i}$, we have $cl_{M(\mathcal{C})}(\{x\})=A_{i}$.
Hence, $\{A_{1},A_{2},\cdots,A_{s}\}\bigcup \{\{x\}:{x\in B}\}$ is the set of atoms of lattice $\mathcal{L}(M(\mathcal{C}))$.
\end{proof}

The proposition below connects simple matroid and the cardinal number of $A_{i}$.
In fact, a matroid is simple if and only if $\forall 1\leq i\leq s$, $|A_{i}|=1$.

\begin{lemma}
$\forall 1\leq i\leq s$, if $|A_{i}|\geq 2$, then $\forall x,y\in A_{i}$, $x,y$ are parallel .
\end{lemma}

\begin{proof}
According to the definition of $A_{i}$, we may as well suppose $A_{i}=K_{h}-\bigcup_{j=1,j\neq h}^{m}\\K_{j}$, where $1\leq h\leq m$.
For all $x,y\in A_{i}$, then $x,y\in K_{h}$ and $\forall 1\leq j\leq m, j\neq h$, $x\notin K_{j}$ and $y\notin K_{j}$.
Thus $\{x,y\}$ is not an independent set.
Based on the definition of transversal matroid and the fact that $\mathcal{C}$ is a covering, any single-point set is an independent set.
Thus $\{x\}$ or $\{y\}$ is an independent set.
Hence, $x,y$ are parallel.
\end{proof}

\begin{proposition}
Let $\mathcal{C}$ be a covering and $M(\mathcal{C})$ the transversal matroid induced by $\mathcal{C}$.
$M(\mathcal{C})$ is a simple matroid if and only if $|A_{i}|=1$ for all $1 \leq i\leq s$.
\end{proposition}

\begin{proof}
"$\Rightarrow$": Since $M(\mathcal{C})$ is a simple matroid, it dose not contain parallel elements.
If there exists $1\leq i\leq s$ such that $|A_{i}|\neq 1$, then $|A_{i}|\geq 2$ because $A_{i} \neq \emptyset$.
According to Lemma 6, $\forall x,y\in A_{i}$, $x,y$ are parallel which contradicts the assumption that $M(\mathcal{C})$ is a simple matroid.
Hence, $\forall 1 \leq i\leq s$, $|A_{i}|=1$.

"$\Leftarrow$": According to the definition of parallel element, if $|A_{i}|=1$, then $M(\mathcal{C})$ does not contain parallel elements.
If $M(\mathcal{C})$ has parallel elements, we may as well suppose $x,y$ are parallel, then there exists only one block which contains $x,y$.
Hence, there exists $1\leq i\leq s$ such that $x,y\in A_{i}$, that is, $|A_{i}|\geq 2$.
This contradicts the fact that $\forall 1 \leq i\leq s$, $|A_{i}|=1$.
Based on the definition of transversal matroid  and the fact that $\mathcal{C}$ is a covering, any single-point set is an independent set, thus $M(\mathcal{C})$ dose not contain loops.
Hence, $M(\mathcal{C})$ dose not contain parallel elements and loops which implies that $M(\mathcal{C})$ is a simple matroid.
\end{proof}

The following two corollaries show that we also have the above results when a covering degenerates into a partition.

\begin{corollary}
Let $\mathcal{P}=\{P_{1},P_{2},\cdots,P_{m}\}$ be a partition.
$\mathcal{P}$ is the set of atoms of lattice $\mathcal{L}(M_{\mathcal{P}})$.
\end{corollary}

\begin{corollary}
Let $\mathcal{P}=\{P_{1}, P_{2}, \cdots, P_{m}\}$ be a partition of $E$ and $M(\mathcal{P})$ the transversal matroid induced by $\mathcal{P}$. $M(\mathcal{P})$ is a simple matroid if and only if $|P_{i}|=1$.
\end{corollary}

For a geometric lattice $\mathcal{L}(M(\mathcal{C}))$, its atoms are the set of closure of single-points.
However, the closure of any two elements of $E$ may not be the set which covers atoms in this lattice.
The following proposition shows in what condition $cl_{M(\mathcal{C})}(\{x,y\})$ covers atoms of lattice $\mathcal{L}(M(\mathcal{C}))$.

\begin{proposition}
For all $x,y \in E$, $cl_{M(\mathcal{C})}(\{x,y\})$ covers $cl_{M(\mathcal{C})}(\{x\})$ if and only if there dose not exist $1 \leq i\leq s$ such that $x,y \in A_{i}$.
\end{proposition}

\begin{proof}
"$\Leftarrow$": $\forall x,y\in E$, from $\{x\}\subseteq \{x,y\}$, we can obtain $cl_{M(\mathcal{C})}(\{x\})\subseteq cl_{M(\mathcal{C})}(\{x,y\})$ and $1=r_{M(\mathcal{C})}(cl_{M(\mathcal{C})}(\{x\}))\leq r_{M(\mathcal{C})}(cl_{M(\mathcal{C})}(\{x,y\}))=r_{M(\mathcal{C})}(\{x,y\})\leq |\{x,y\}|=2$. Now we need to prove $r_{M(\mathcal{C})}(cl_{M(\mathcal{C})}(\{x,y\}))$ $=2$.
If $r_{M(\mathcal{C})}(cl_{M(\mathcal{C})}(\{x,y\}))=r_{M(\mathcal{C})}(\{x,y\})=1=r_{M(\mathcal{C})}(\{x\})$, then $y\in cl_{M(\mathcal{C})}(\{x\})$, thus $x,y \notin \mathcal{I}(\mathcal{C})$, that is, there is only one block contains $x,y$.
It means that there exists $A_{i}$ such that $x,y\in A_{i}$. That contradicts the hypothesis.
Hence, $r_{M(\mathcal{C})}(cl_{M(\mathcal{C})}(\{x,y\}))=2$, that is, $cl_{M(\mathcal{C})}(\{x,y\})$ covers $cl_{M(\mathcal{C})}(\{x\})$.

"$\Rightarrow$": $\forall x,y \in E$, if there exists $A_{i}$ such that $x,y\in A_{i}$, then there is only one block contains $x,y$, thus $x,y\notin \mathcal{I}(\mathcal{C})$, hence $\{x,y\}\subseteq cl_{M(\mathcal{C})}(\{x\})$.
That implies $cl_{M(\mathcal{C})}(\{x,y\})\subseteq cl_{M(\mathcal{C})}(\{x\})$ which contradicts the assumption that $cl_{M(\mathcal{C})}(\{x,y\})$ covers $cl_{M(\mathcal{C})}(\{x\})$.
\end{proof}

The modular element and the modular pair are core concepts in lattice.
The following theorem shows the relationship among modular element, modular pair and rank function of a matriod.

\begin{theorem}
Let $M$ be a matriod and $\mathcal{L}(M)$ the set of all closed sets of $M$.\\
(1) For all $X,Y\in \mathcal{L}(M)$, $X$ and $Y$ is a modular pair if and only if $r_{M}(X\bigcup Y)+r_{M}(X\bigcap Y)=r_{M}(X)+r_{M}(Y)$.\\
(2) For all $X\in \mathcal{L}(M)$, $X$ is a modular element of $\mathcal{L}(M)$ if and only if $r_{M}(X\bigcup Y)+r_{M}(X\bigcap Y)=r_{M}(X)+r_{M}(Y)$, for all $Y\in \mathcal{L}(M)$.
\end{theorem}

\begin{proof}
(1) According to Lemma 1 and Lemma 3, we know $(X,Y)$ is a modular pair if and only if $r_{M}(X)+r_{M}(Y)=r_{M}(X\bigvee Y)+r_{M}(X\bigwedge Y)=r_{M}(cl_{M}(X\bigcup Y))+r_{M}(X\bigcap Y)=r_{M}(X\bigcup Y)+r_{M}(X\bigcap Y)$.

(2) It comes from the definition of modular element and (1).
\end{proof}

Let $\{A_{i}:i\in \Gamma\}$ be the set of atoms of lattice $\mathcal{L}(M(\mathcal{C}))$, where $\Gamma$ denotes the index set.
The following theorem shows the relationship among atoms, modular pairs and modular elements.

\begin{theorem}
Let $\mathcal{C}$ be a covering and $\mathcal{L}(M(\mathcal{C}))$ the set of all closed sets of transversal matroid $M(\mathcal{C})$ induced by $\mathcal{C}$.
For all $i,j\in \Gamma$\\
(1) $(A_{i},A_{j})$ is a modular pair of $\mathcal{L}(M(\mathcal{C}))$.\\
(2) $A_{i}$ is a modular element of $\mathcal{L}(M(\mathcal{C}))$.
\end{theorem}

\begin{proof}
(1) Since $\mathcal{C}$ is a covering, $cl_{M(\mathcal{C})}(\emptyset)=\emptyset$.
$A_{i}$ and $A_{j}$ are atoms, so $A_{i}\bigcap A_{j}=\emptyset$.
According to Theorem 3, we need to prove $r_{M(\mathcal{C})}(A_{i}\bigcup A_{j})+r_{M(\mathcal{C})}(A_{i}\bigcap A_{j})=r_{M(\mathcal{C})}(A_{i})+r_{M(\mathcal{C})}(A_{j})$, that is, $r_{M(\mathcal{C})}(A_{i}\bigcup A_{j})=2$.
According to the submodular inequality of $r_{M(\mathcal{C})}$, we have $r_{M(\mathcal{C})}(A_{i}\bigcup A_{j})+r_{M(\mathcal{C})}(A_{i}\bigcap A_{j})\leq r_{M(\mathcal{C})}(A_{i})+r_{M(\mathcal{C})}(A_{j})$, that is, $1 \leq r_{M(\mathcal{C})}(A_{i}\bigcup A_{j})\leq 2$.
If $r_{M(\mathcal{C})}(A_{i}\bigcup A_{j})=1=r_{M(\mathcal{C})}(A_{i})$, then $A_{j}\subseteq cl_{M(\mathcal{C})}(A_{i})=A_{i}$ which contradicts that $A_{i}\bigcap A_{j}=\emptyset$.\\
(2) $A_{i}$ is a modular element of $\mathcal{L}(M(\mathcal{C}))$ if and only if $r_{M(\mathcal{C})}(A_{i}\bigcup A)+r_{M(\mathcal{C})}(A_{i}\bigcap A)\\=r_{M(\mathcal{C})}(A_{i})+r_{M(\mathcal{C})}(A)$ for all $A\in \mathcal{L}(M(\mathcal{C}))$.

Case 1: If $A_{i}$ and $A$ are comparable, that is, $A_{i}\subseteq A$, then $r_{M(\mathcal{C})}(A_{i}\bigcup A)+r_{M(\mathcal{C})}(A_{i}\bigcap A)=r_{M(\mathcal{C})}(A_{i})+r_{M(\mathcal{C})}(A)$.

Case 2: If $A_{i}$ and $A$ are not comparable, there are two cases.
One is that $A$ is an atom of $\mathcal{L}(M(\mathcal{C}))$, the other is that $A$ is not an atom of $\mathcal{L}(M(\mathcal{C}))$.
If $A$ is a atom of $\mathcal{L}(M(\mathcal{C}))$, then we obtain the result from (1).
If $A$ is not an atom of $\mathcal{L}(M(\mathcal{C}))$, then $A\bigcap A_{i}=\emptyset$.
Hence, $r_{M(\mathcal{C})}(A)\leq r_{M(\mathcal{C})}(A\bigcup A_{i})\leq r_{M(\mathcal{C})}(A)+1$.
If $r_{M(\mathcal{C})}(A\bigcup A_{i})= r_{M(\mathcal{C})}(A)$, then $A_{i}\subseteq cl_{M(\mathcal{C})}(A)=A$ which contradicts that $A_{i}\bigcap A=\emptyset$.
Hence, $r_{M(\mathcal{C})}(A\bigcup A_{i})=r_{M(\mathcal{C})}(A)+1$.

In a word, for all $A\in \mathcal{L}(M(\mathcal{C}))$, $r_{M(\mathcal{C})}(A_{i}\bigcup A)+r_{M(\mathcal{C})}(A_{i}\bigcap A)=r_{M(\mathcal{C})}(A_{i})+r_{M(\mathcal{C})}(A)$, that is, $A_{i}$ is a modular element of $\mathcal{L}(M(\mathcal{C}))$.
\end{proof}

\begin{corollary}
Let $\mathcal{L}(M(\mathcal{P}))$ be the set of all closed sets of transversal matroid induced by $\mathcal{P}$.
For all $P_{i},P_{j}\in \mathcal{P}$\\
(1) $(P_{i},P_{j})$ is a modular pair of $\mathcal{L}(M(\mathcal{P}))$.\\
(2) $P_{i}$ is a modular element of $\mathcal{L}(M(\mathcal{P}))$.
\end{corollary}

The following lemma shows how to induce a matroid by a lattice.
In fact, if a function $f$ on a lattice is a non-negative, inter-valued, submodular and  $f(\emptyset)=0$, then it can determine a matroid.

\begin{lemma}\cite{Oxley93Matroid}
Let $\mathcal{L}_{E}$ be a lattice such that $\mathcal{L}_{E}$ is closed under intersection, and contains $\emptyset$ and $E$.
Suppose that $f$ is a non-negative, inter-valued, submodular function on $\mathcal{L}_{E}$ for which $f(\emptyset)=0$.
Let $\mathcal{I}(\mathcal{L}_{E},f)=\{X\subseteq E:f(X)\geq |X\bigcap T|, \forall T\in \mathcal{L}_{E}\}$.
$\mathcal{I}(\mathcal{L}_{E},f)$ is the collection of independent sets of a matroid on $E$.
\end{lemma}

According to the definition of $\mathcal{L}(M(\mathcal{C}))$, we find that $\mathcal{L}(M(\mathcal{C}))$ is closed under intersection, and contains $\emptyset$ and $E$.
Moreover, the rank function is a non-negative, inter-valued, submodular function on $\mathcal{L}(M(\mathcal{C})$ for which $r_{M(\mathcal{C})}(\emptyset)=0$.
Similar to Lemma 7, we can obtain the following theorem.

\begin{theorem}
Let $\mathcal{L}(M(\mathcal{C}))$ be the set of all closed sets of transversal matroid induced by $\mathcal{C}$.
We define $\mathcal{I}(\mathcal{L}(M(\mathcal{C})), r_{M(\mathcal{C})})=\{X\subseteq E:r_{M(\mathcal{C})}(Y)\geq |X\bigcap Y|,\forall Y\in \mathcal{L}(M(\mathcal{C}))\}$, then $M(E,\mathcal{I}(\mathcal{L}(M(\mathcal{C}),r_{M(\mathcal{C})}))$ is a matriod.
\end{theorem}

From above theorem, we find that the matroid induced by a geometric lattice is the same as the transversal matroid which generates the geometric lattice. While the concert form of rank function of transversal matroid is difficult to be expressed.
The following theorem solves this problem.

\begin{theorem}
Let $\mathcal{L}(M(\mathcal{C}))$ be the set of all closed sets of transversal matriod induced by covering $\mathcal{C}$ and $r_{\mathcal{L}(M(\mathcal{C}))}$ the rank function of $M(E,\mathcal{I}(\mathcal{L}(M(\mathcal{C}),r_{M(\mathcal{C})}))$. $r_{\mathcal{L}(M(\mathcal{C}))}\\(X)=min_{Y\in \mathcal{L}(M_{\mathcal{C}})}(r_{M(\mathcal{C})}(X)+|X-Y|)$.
\end{theorem}

\begin{proof}
$X\in \mathcal{I}(\mathcal{L}(M(\mathcal{C}), r_{M(\mathcal{C})}) \Leftrightarrow \forall Y\in \mathcal{L}(M(\mathcal{C}))$, $r_{M(\mathcal{C})}(Y)\geq |X\bigcap Y|=|X|-|X-Y| \Leftrightarrow \forall Y\in \mathcal{L}(M(\mathcal{C}))$, $|X|\leq r_{M(\mathcal{C})}(Y)+|X-Y| \Leftrightarrow \forall Y\in \mathcal{L}(M(\mathcal{C})$, $r_{\mathcal{L}(M(\mathcal{C}))}(X)=|X|\leq r_{M(\mathcal{C})}(Y)+|X-Y| \Leftrightarrow r_{\mathcal{L}(M(\mathcal{C}))}(X)=min_{Y\in \mathcal{L}(M(\mathcal{C}))}(r_{M(\mathcal{C})}\\(X)+|X-Y|)$.
\end{proof}

For any given matroid $M$, we know that for all $X\subseteq E$, $X$ is an independent set if and only if $r_{M}(X)=|X|$.
Based on the properties of rank function, we have $r_{M}(X)\leq |X|$.
Hence, $X$ is an independent set if and only if $r_{M}(X)\geq |X|$.

\begin{lemma}
Let $M$ be a matroid.
$X$ is an independent set of $M$ if and only if for all closed set of $Y$, $r_{M}(Y)\geq |X\bigcap Y|$.
\end{lemma}

\begin{proof}
"$\Rightarrow$": Since $X\bigcap Y\subseteq X$ for all closed set $Y$ and $X$ is an independent set, $X\bigcap Y$ is an independent set of $M$.
Hence, $r_{M}(Y)\geq r_{M}(X\bigcap Y)\geq |X\bigcap Y|$.

"$\Leftarrow$": For all closed set $Y$, $r_{M}(Y)\geq |X\bigcap Y|$. Especially, for $Y=cl_{M}(X)$, we have $r_{M}(X)=r_{M}(cl_{M}(X))=r_{M}(Y)\geq |X\bigcap Y|=|X|$.
Hence, $X$ is an independent set of matroid $M$.
\end{proof}

The following theorem shows that there is a one-to-one correspondence between geometric lattices and transversal matroids in the context of covering-based rough sets.

\begin{theorem}
Let $r_{\mathcal{L}(M(\mathcal{C}))}$ be the rank function of $M(E,\mathcal{I}(\mathcal{L}(M(\mathcal{C}),r_{M(\mathcal{C})}))$. $\mathcal{I}(\mathcal{L}(M(\mathcal{C}),\\r_{M(\mathcal{C})})=\mathcal{I}(M(\mathcal{C}))$ and $r_{\mathcal{L}(M(\mathcal{C}))}(X)=r_{M(\mathcal{C})}(X)$ for all $X\subseteq E$.
\end{theorem}

\begin{proof}
According to Lemma 8, we know that $\mathcal{I}(\mathcal{L}(M(\mathcal{C}), r_{M(\mathcal{C})})=\{X\subseteq E:r_{M(\mathcal{C})}(Y)\\\geq |X\bigcap Y|,\forall Y\in \mathcal{L}(M(\mathcal{C}))\}$ and $\mathcal{I}(\mathcal{C})=\{X\subseteq E:r_{M(\mathcal{C})}(X)= |X|\}$ are equivalent, that is, $M(E,\mathcal{I}(\mathcal{L}(M(\mathcal{C}),r_{M(\mathcal{C})}))$ and $M(E,\mathcal{I}(\mathcal{C}))$ are equivalent. So does $r_{\mathcal{L}(M(\mathcal{C}))}(X)=r_{M(\mathcal{C})}(X)$ for all $X\subseteq E$.
\end{proof}

We know that $\mathcal{I}(\mathcal{P})=\{X\subseteq E:|X\bigcap P_{i}|\leq 1\}$ is the family of all partial transversal of partition $\mathcal{P}$. For all $X\subseteq E$, $r_{M(\mathcal{P})}(X)=max \{|I|: I\subseteq X, I\in \mathcal{I}(\mathcal{P})\}=|\{P_{i}: P_{i}\bigcap X\neq \emptyset\}|$.   According to Lemma 2, we know that for all closed set $X,Y$ of matroid $M$, $X\bigvee Y=cl_{M}(X\bigcup Y)$.
If the matroid is $M(\mathcal{P})$, then $X\bigwedge Y=X\bigcup Y$.

\begin{lemma}\cite{ZhuWang11Matroidal}
If $R$ is an equivalence relation on $E$ and $M$ is the matroid, then $cl_{M}=R^{\ast}$ for all $X\subseteq E$.
\end{lemma}

\begin{lemma}
Let $M(\mathcal{P})$ be a transversal matroid induced by $\mathcal{P}$ and $\mathcal{L}(M(\mathcal{P}))$ the set of all closed sets of transversal matriod induced by $\mathcal{P}$, $X,Y\in \mathcal{L}(M(\mathcal{P}))$. $X\bigvee Y=X\bigcup Y$.
\end{lemma}

\begin{proof}
$X\bigvee Y=cl_{M(\mathcal{P})}(X\bigcup Y)=R^{\ast}(X\bigcup Y)=R^{\ast}(X)\bigcup R^{\ast}(Y)=cl_{M(\mathcal{P})}(X)\bigcup\\ cl_{M(\mathcal{P})}(Y)=X\bigcup Y$.
\end{proof}

\begin{proposition}
Let $r_{\mathcal{L}(M(\mathcal{P}))}$ be the rank function of $M(E,\mathcal{I}(\mathcal{L}(M(\mathcal{P}),r_{M(\mathcal{P})}))$. $\mathcal{I}(\mathcal{L}\\(M(\mathcal{P}),r_{M(\mathcal{P})})=\mathcal{I}(M(\mathcal{P}))$ and $r_{\mathcal{L}(M(\mathcal{P}))}(X)=r_{M(\mathcal{P})}(X)$ for all $X\subseteq E$.
\end{proposition}

\begin{proof}
We need to prove only $\mathcal{I}(\mathcal{L}(M(\mathcal{P}),r_{M(\mathcal{P})})=\mathcal{I}(M(\mathcal{P}))$.
$\forall X\in \mathcal{I}(\mathcal{L}(M(\mathcal{P}),\\r_{M(\mathcal{P})})$, $r_{M(\mathcal{P})}(Y)\geq |X\bigcap Y|$ for all $Y\in \mathcal{L}(M(\mathcal{P}))$.
Since $P_{i}\in \mathcal{L}(M(\mathcal{P}))$, $|X\bigcap P_{i}|\\\leq r_{M(\mathcal{P})}(P_{i})=1$. Thus $X\in \mathcal{I}(M(\mathcal{P}))$.
Hence, $\mathcal{I}(\mathcal{L}(M(\mathcal{P}),r_{M(\mathcal{P})})\subseteq \mathcal{I}(M(\mathcal{P}))$.
According to Lemma 4 and Lemma 10, for all $Y\in \mathcal{L}(M(\mathcal{P}))$, there exists $K\subseteq \{1,2,\cdots m\}$ such that $Y=\bigvee_{i\in K}P_{i}=\bigcup_{i\in K}P_{i}$, $r_{M(\mathcal{P})}(Y)=|K|$ and $|X\bigcap Y|=|X\bigcap (\bigcup_{i\in K}P_{i})|=|\bigcup_{i\in K}(X\bigcap P_{i})|=\sum_{i\in K}|X\bigcap P_{i}|$.
For all $X\in \mathcal{I}(M(\mathcal{P}))$, $|X\bigcap P_{i}|\leq 1$, $|X\bigcap Y|=\sum_{i\in K}|X\bigcap P_{i}|\leq |K|=r_{M(\mathcal{P})}(Y)$, that is, $\mathcal{I}(M(\mathcal{P}))\subseteq \mathcal{I}(\mathcal{L}(M(\mathcal{P}),r_{M(\mathcal{P})})$.
\end{proof}

\section{Three geometric lattice structures of covering-based rough sets through approximation operators}
\label{S:Matroidalstructuresbasedonthreetypesofupperapproximations}

A geometric lattice structure of covering-based rough sets is established through the transversal matroid induced by the covering, and its
characteristics including atoms, modular elements and modular pairs are studied in Section~\ref{S:Thepropertiesoftheclosedsetlatticeofmatroidinducedbycovering}.
In this section, we study the matroid and the structure of geometric lattice of a matroid from the viewpoint of upper approximations.
The condition of three types of upper approximations to be a matroidal closure operator is obtained and the properties of matroids and their geometric lattice structures induced by them are also established.

Pomykala first studied the second type of covering rough set model~\cite{Pomykala87Approximation}. Zhu and Wang studied the axiomatization of this type of approximation and the relationship between it and the $Kuratowski$ closure operator in \cite{ZhuWang07OnThree}.
Proposition 9 below gives some properties of this operator.
\begin{proposition}
Let $\mathcal{C}$ be a covering on $E$. $SH$ has the following properties:\\
(1) $SH(\emptyset)=\emptyset$\\
(2) $X\subseteq SH(X)$ for all $X\subseteq E$\\
(3) $SH(X\bigcup Y)=SH(X)\bigcup SH(Y)$\\
(4) $x\in SH(\{y\})\Leftrightarrow y\in SH(\{x\})$\\
(5) $X\subseteq Y\subseteq E\Rightarrow SH(X)\subseteq SH(Y)$\\
(6) $\forall x,y \in E, y \in SH(X\bigcup \{x\})-SH(X)$, then $x\in SH(X\bigcup \{y\})$.
\end{proposition}

\begin{proof}
$(1)-(5)$ were shown in~\cite{Pomykala87Approximation,Zakowski83Approximations,ZhangLiWu10OnAxiomatic}.
Here we prove only (6).
According to (3), we know $SH(X\bigcup \{x\})-SH(X)=SH(X)$ $\bigcup SH(\{x\})-SH(X)=SH(\{x\})-SH(X)$.
If $y\in SH(X\bigcup \{x\})-SH(X)$, then $y\in SH(\{x\})$.
According to (4) and (5), we have $x\in SH(\{y\})\subseteq SH(X\bigcup \{y\})$.
\end{proof}

We find that the idempotent of $SH$ is not valid, so what is the conditions that guarantee it holds for $SH$?
We have the following conclusion.

\begin{proposition}
Let $\mathcal{C}$ be a covering. $SH(SH(X))=SH(X)$ if and only if $\{I(x):x\in E\}$ induced by $\mathcal{C}$ forms a partition.
\end{proposition}

\begin{proof}
"$\Leftarrow$": According to (2), (5) of Proposition 9, we have $SH(X)\subseteq SH(SH(X))$.
Now we prove $SH(SH(X))\subseteq SH(X)$.
For all $x\in SH(SH(X))$, there exists $y\in SH(X)$ such that $x\in I(y)$.
Since $y\in SH(X)$, there exists $z\in X$ such that $y\in I(z)$.
According to the definition of $I(y)$, we know $y\in I(y)$, thus $I(z)\bigcap I(y)\neq \emptyset$.
For $\{I(x):x\in E\}$ forms a partition, $I(z)=I(y)$.
Since $x\in I(y)$, $x\in I(z)$, that is, $x\in SH(X)$, thus $SH(SH(X))\subseteq SH(X)$.

"$\Rightarrow"$: In order to prove $\{I(x):x\in E\}$ forms a partition, we need to prove that for all $x, y \in E$, if $I(x)\bigcap I(y)\neq \emptyset$, then $I(x)=I(y)$.
If $I(x)\bigcap I(y)\neq \emptyset$, then there exists $z\in I(x)\bigcap I(y)$.
For $SH(SH(\{x\}))=\bigcup\{I(u):u\in I(x)\}$ and $z\in I(x)$, then $I(z)\subseteq SH(SH(\{x\}))=SH(\{x\})=I(x)$.
Based on the definition of $I(z)$ and $z\in I(x)$, we have $x\in I(z)$, thus $I(x)\subseteq SH(SH(\{z\}))=SH(\{z\})=I(z)$.
Hence, $I(x)=I(z)$.
Similarly, we can obtain $I(y)=I(z)$, thus $I(x)=I(z)=I(y)$.
\end{proof}

From the above proposition, it is easy to obtain the following theorem.

\begin{theorem}
Let $\mathcal{C}$ be a covering. $SH$ is a closure operator of matriod if and only if $\{I(x):x\in E\}$ induced by $\mathcal{C}$ forms a partition.
\end{theorem}

\begin{proof}
It comes from (2),(5) and (6) of Proposition 2, 9 and 10.
\end{proof}

\begin{definition}
Let $\mathcal{C}$ be a covering. If $\{I(x):x\in E\}$ induced by $\mathcal{C}$ forms a partition, then we define $\mathcal{I}^{'}=\{I\subseteq E: |I\bigcap I(x_{i})|\leq 1,\forall i\in \{1,2,\cdots,s\}\}$.
\end{definition}

As we know, if $\{I(x):x\in E\}$ forms a partition, then $M(E, \mathcal{I}^{'})$ is a matroid.
Under this condition, we know that $SH$ is a closure operator of a matroid, thus it can determine a matroid, and its independent set is established as follows.
\begin{center}
 $\mathcal{I}_{SH}(\mathcal{C})=\{I\subseteq E: \forall x\in I, x\notin SH(I-\{x\})\}$.
\end{center}
The following proposition shows $M(E, \mathcal{I}_{SH})=M(E, \mathcal{I}^{'})$ under the condition that $\{I(x):x\in E\}$ forms a partition.

\begin{proposition}
Let $\mathcal{C}$ be a covering. If $\{I(x):x\in U\}$ induced by $\mathcal{C}$ forms a partition, then $M(E, \mathcal{I}_{SH}(\mathcal{C}))$ is a matroid and $\mathcal{I}_{SH}(\mathcal{C})=\mathcal{I}^{'}$.
\end{proposition}

\begin{proof}
Let $\mathcal{I}_{cl}=\{I\subseteq E: \forall x\in I, x\notin cl(I-\{x\})\}$.
we know that if an operator $cl$ satisfies $(1)-(4)$ of Proposition 2, $M(E, \mathcal{I}_{cl})$ is a matroid.
$\{I(x):x\in U\}$ induced by $\mathcal{C}$ forms a partition, hence, $M(E, \mathcal{I}_{SH}(\mathcal{C}))$ is a matroid.
Since $SH(I)=\bigcup_{y\in I}I(y)$, $SH(I-\{x\})=\bigcup_{y\in I-\{x\}}I(y)$.
According to the definition of $I(x)$, we know $x\in I(x)$.
On one hand, for all $I\in \mathcal{I}_{SH}(\mathcal{C})$, we know that $\forall x\in I$, $x\notin SH(I-\{x\})$, that is, $\forall y\in I$ and $y\neq x$, $x\notin I(y)$.
If $I\notin \mathcal{I}^{'}$, that is, there exists $1 \leq i \leq s$ such that $|I\bigcap I(x_{i})|\geq 2$, then we may as well suppose there exist $u, v$ such that $u,v\in I(x_{i})$ and $u,v\in I$.
Since $u\in I(u), v\in I(v)$ and $\{I(x):x\in E\}$ forms a partition, $I(u)=I(v)=I(x_{i})$.
Based on that, we know there exists $u\in I$ and $u\neq v$ such that $u\in I(v)$, that implies contradiction. Hence, $I\in \mathcal{I}^{'}$, that is, $\mathcal{I}_{SH}(\mathcal{C})\subseteq\mathcal{I}^{'}$.
On the other hand, if $I\notin \mathcal{I}_{SH}(\mathcal{C})$, then there exists $x\in I$ such that $x\in SH(I-\{x\})=\bigcup_{y\in I-\{x\}}I(y)$. That implies that there exists $y\in I$ and $y\neq x$ such that $x\in I(y)$.
Since $x\in I(x)$ and $\{I(x):x\in U\}$ forms a partition,  $I(x)=I(y)$.
Thus $x,y\in I\bigcap I(x)$, that implies $|I\bigcap I(x)|\geq 2$, i.e., $I\notin \mathcal{I}^{'}$.
Hence, $\mathcal{I}^{'}\subseteq \mathcal{I}_{SH}(\mathcal{C})$.
\end{proof}

In the following proposition, we study some properties of $M(E, \mathcal{I}_{SH})$.

\begin{proposition}
Let $\mathcal{C}$ be a covering. If $\{I(x_{1}),I(x_{2}),\cdots,I(x_{s})\}$ induced by $\mathcal{C}$ forms a partition and $M(E,\mathcal{I}_{SH}(\mathcal{C}))$ is the matriod induced by $SH$, then\\
(1) $X$ is a base of $M(E,\mathcal{I}_{SH}(\mathcal{C}))$ if and only if $|X\bigcap I(x_{i})|=1$ for all $i\in \{1,2,\cdots,s\}$. Moreover, $M(E,\mathcal{I}_{SH}(\mathcal{C}))$ has $|I(x_{1})||I(x_{2})|\cdots |I(x_{s})|$ bases.\\
(2) For all $X\subseteq E$, $r_{SH}(X)=|\{I(x_{i}):I(x_{i})\bigcap X\neq \emptyset,i=1,2,\cdots,s\}|$.\\
(3) $X$ is a dependent set of $M(E,\mathcal{I}_{SH}(\mathcal{C}))$ if and only if there exists $I(x_{i})$ such that $|I(x_{i})\bigcap X|>1$.\\
(4) $X$ is a circuit of $M(E,\mathcal{I}_{SH}(\mathcal{C}))$ if and only if there exists $I(x_{i})$ such that $X\subseteq I(x_{i})$ and $|X|=2$.
\end{proposition}

\begin{proof}
(1) According to the definition of base of a matroid, we know that $X$ is a base of $M(E,\mathcal{I}_{SH}(\mathcal{C}))$ $\Leftrightarrow$ $X\in max(\mathcal{I}_{SH}(\mathcal{C}))$ $\Leftrightarrow$ $|X\bigcap I(x_{i})|=1$ for all $i\in \{1,2,\cdots,s\}$ according to the definition of $\mathcal{I}_{SH}(\mathcal{C})$.
Since $X$ is a base of $M(E,\mathcal{I}_{SH}(\mathcal{C}))$ and $\{I(x_{1}),\\I(x_{2}), \cdots, I(x_{s})\}$ are different, $M(E,\mathcal{I}_{SH}(\mathcal{C}))$ has $|I(x_{1})||I(x_{2})|\cdots |I(x_{s})|$ bases.

(2) According to the definition of rank function, we know $r_{SH}(X)=|B_{X}|=|\{I(x_{i}):|B_{X}\bigcap I(x_{i})|=1\}|\leq |\{I(x_{i}): X\bigcap I(x_{i})\neq \emptyset\}|$, where $B_{X}$ is a maximal independent set included in $X$.
Now we just need to prove the inequality $|\{I(x_{i}):|B_{X}\bigcap I(x_{i})|=1\}|< |\{I(x_{i}): X\bigcap I(x_{i})\neq \emptyset\}|$ dose not hold;
otherwise, there exists $1\leq i\leq s$ such that $I(x_{i})\bigcap X\neq\emptyset$ and $I(x_{i})\bigcap B_{X}=\emptyset$.
Thus there exists $e_{i}\in I(x_{i})\bigcap X$ such that $B_{X}\bigcup \{e_{i}\}\subseteq X$ and $B_{X}\bigcup \{e_{i}\}\in \mathcal{I}_{SH}(\mathcal{C})$.
That contradicts the assumption that $B_{X}$ is a maximal independent set included in $X$.
Hence, $r_{SH}(X)=|\{I(x_{i}):I(x_{i})\bigcap X\neq \emptyset,i=1,2,\cdots,s\}|$.

(3) According to the definition of dependent set, we know that $X$ is a dependent set $\Leftrightarrow$ $X\notin \mathcal{I}_{SH}(\mathcal{C})$ $\Leftrightarrow$ there exists $1 \leq i\leq s$ such that $|X\bigcap I(x_{i})|> 1$.

(4) "$\Rightarrow$": As we know, a circuit is a minimal dependent set.
$X$ is a circuit of $M(E,\mathcal{I}_{SH}(\mathcal{C}))$, there exists $1\leq i\leq s$ such that $|I(x_{i})\bigcap X|=2$.
Now we just need to prove $|X|=2$; otherwise, we may as well suppose $X=\{x,y,z\}$ where $x,y\in I(x_{i})\bigcap X$.
Thus we can obtain $|(X-\{z\})\bigcap I(x_{i})|=2$, that is, $X-\{z\}\notin \mathcal{I}_{SH}(\mathcal{C})$.
That contradicts the minimality of circuit.
Combining $|X|=2$ with $|I(x_{i})\bigcap X|=2$, we have $X\subseteq I(x_{i})$.

"$\Leftarrow$": Since $|X|=2$, we may as well suppose $X=\{x,y\}$.
$|X\bigcap I(x_{i})|=2$ because there exists $I(x_{i})$ such that $X\subseteq I(x_{i})$, thus $X$ is a dependent set.
For all $1\leq j\leq s$, $|\{x\}\bigcap I(x_{j})|\leq 1$ and $|\{y\}\bigcap I(x_{j})|\leq 1$.
That implies $\{x\}$ and $\{y\}$ are independent sets, hence $X$ is a circuit of $M(E,\mathcal{I}_{SH}(\mathcal{C}))$.
\end{proof}

We denote $\mathcal{L}_{SH}(M(\mathcal{C}))$ as the set of all closed sets of $M(U,\mathcal{I}_{SH}(\mathcal{C}))$.
When $\{I(x_{1}),\\
I(x_{2}), \cdots,I(x_{s}\}$ forms a partition, for all $X,Y\in \mathcal{L}_{SH}(M(\mathcal{C}))$, $X\bigwedge Y=X\bigcap Y$, $X\bigvee Y=SH(X\bigcup Y)=SH(X)\bigcup SH(Y)=X\bigcup Y$ and $SH(\emptyset)=\emptyset$.

\begin{proposition}
If $\{I(x_{1}),I(x_{2}),\cdots,I(x_{s}\}$ forms a partition, then\\
(1) $\{I(x_{1}),I(x_{2}),\cdots,I(x_{s})\}$ are all atoms of $\mathcal{L}({\mathcal{M}_{SH}})$.\\
(2) There dose not exist $z\in U$ such that $x,y\in I(z)$ if and only if $SH(\{x,y\})$ covers $SH(\{x\})$ or $SH(\{y\})$.\\
(3) $\forall 1\leq i,j \leq s,(I(x_{i}),I(x_{j}))$ is a modular pair of $\mathcal{L}_{SH}(M(\mathcal{C}))$.\\
(4) $\forall 1\leq i\leq s,I(x_{i})$ is a modular element of $\mathcal{L}_{SH}(M(\mathcal{C}))$.
\end{proposition}

\begin{proof}
(1) comes from Corollary 1, Theorem 8 and Proposition 11.
Based on Proposition 4 and Theorem 8, we can obtain (2).
According to Corollary 3, Theorem 8 and proposition 11, it is easy to obtain (3) and (4).
\end{proof}

Based on Theorem 8, we know that the condition on which $SH$ becomes a closure operator of a matroid is that $\{I(x):x\in E\}$ forms a partition.
Proposition 14 and 15 below show what kinds of coverings can satisfy this condition.

\begin{lemma}
Let $\mathcal{C}$ be a covering, $K\in \mathcal{C}$.
If $K$ is an immured element, then $I(x)$ is the same in $\mathcal{C}$ as in $\mathcal{C}-\{K\}$.
\end{lemma}

\begin{proof}
If $x\notin K$, then $I_{\mathcal{C}}(x)=I_{\mathcal{C}-\{K\}}(x)$.
If $x\in K$, then $I_{\mathcal{C}}(x)=\bigcup_{x\in K^{'}}K^{'}\bigcup K$.
Since $K$ is an immured element, there exists $x\in K^{'}$ such that $K\subseteq K^{'}$.
Thus $I_{\mathcal{C}}(x)=\bigcup_{x\in K^{'}}K^{'}\bigcup K=\bigcup_{x\in K^{'}}K^{'}=I_{\mathcal{C}-\{K\}}(x)$.
Hence, $I(x)$ is the same in $\mathcal{C}$ as in $\mathcal{C}-\{K\}$.
\end{proof}

\begin{proposition}
Let $\mathcal{C}$ be a covering.
If $exclusion(\mathcal{C})$ is a partition, then $\{I(x):x\in E\}$ induced by $\mathcal{C}$ forms a partition.
\end{proposition}

\begin{proof}
Since $exclusion(\mathcal{C})$ is a partition of $E$, $\{I(x):x\in E\}$ induced by $exclusion(\mathcal{C})$ forms a partition.
Suppose $\{K_{1},K_{2},\cdots, K_{s}\}$ is the set all immured elements of $\mathcal{C}$.
According to Lemma 11, we have $\forall x\in E$, $I(x)$ is the same in $exclusion(\mathcal{C})$ as in $exclusion(\mathcal{C})\bigcup \{K_{1}\}$.
Thus $\{I(x):x\in E\}$ induced by $exclusion(\mathcal{C})\bigcup\{K_{1}\}$ forms a partition.
And the rest may be deduced by analogy, we know that $\forall x\in E$, $I(x)$ is the same in $exclusion(\mathcal{C})$ as in $\mathcal{C}$, thus $\{I(x):x\in E\}$ induced by $\mathcal{C}$ forms a partition.
\end{proof}

The proposition below establishes the necessary and sufficient condition of $\{I(x):x\in E\}$ forms a partition.

\begin{proposition}
$\{I(x):x\in E\}$ induced by $\mathcal {C}$ forms a partition if and only if $\mathcal{C}$ satisfies
$(TRA)$ condition: $\forall x,y,z \in E$, $x,z\in K_{1}\in \mathcal{C}$, $y,z\in K_{2}\in \mathcal{C}$, there exists $K_{3}\in \mathcal{C}$ such that $x,y\in K_{3}$.
\end{proposition}

\begin{proof}
"$\Leftarrow$": $\forall x,y\in E$, $I(x)\bigcap I(y)=\emptyset$ or $I(x)\bigcap I(y)\neq \emptyset$.
If $I(x)\bigcap I(y)\neq \emptyset$, then there exists $z\in I(x)$ and $z\in I(y)$.
According to the definition of $I(x)$ and $I(y)$, there exist $K_{1}, K_{2}$ such that $x,z\in K_{1}$ and $y,z\in K_{2}$.
According to hypothesis, we know $\exists K_{3}\in \mathcal{C}$ such that $x,y\in K_{3}$.
Now we need to prove only $I(x)=I(y)$. $\forall u\in I(x)$, there exists $K$ such that $u, x\in K$.
Since $x,y\in K_{3}$, there exists $K^{'}$ such that $u,y\in K^{'}$, that is, $u\in I(y)$, thus $I(x)\subseteq I(y)$.
Similarly, we can prove $I(y)\subseteq I(x)$.
Hence, $I(x)=I(y)$, that is, $\{I(x):x\in E\}$ forms a partition.

"$\Rightarrow$": $\forall x,y,z \in E$, $x,z\in K_{1}\in \mathcal{C}$ and $y,z\in K_{2}\in \mathcal{C}$, we can obtain $z\in I(x)$ and $z\in I(y)$. That implies $I(x)\bigcap I(y)\neq \emptyset$.
Since $\{I(x):x\in E\}$ forms a partition, $I(x)=I(y)$.
Thus there exists $K_{3}\in \mathcal{C}$ such that $x,y\in K_{3}$.
\end{proof}

\begin{theorem}
$\mathcal{C}$ satisfies $(TRA)$ condition if and only if $SH$ induced by $\mathcal{C}$ is a closure operator of a matroid.
\end{theorem}

\begin{proof}
It comes from Theorem 8 and Proposition 15.
\end{proof}
The sixth type of covering-based upper approximation was first defined in~\cite{XuWang05On}.
Xu and Wang introduced this type of covering-based rough set model and studied the relationship between it and binary relation based rough set model. Proposition 16 below gives some properties of this covering upper approximation operator.

\begin{proposition}\cite{Zhu09RelationshipBetween}
Let $\mathcal{C}$ be a covering of $U$. $XH$ has the following properties:\\
(1) $XH(U)=U$\\
(2) $XH(\emptyset)=\emptyset$\\
(3) $X\subseteq XH(X)$\\
(4) $XH(X\bigcup Y)=XH(X)\bigcup XH(Y)$\\
(5) $XH(XH(X))=XH(X)$\\
(6) $X\subseteq Y \Rightarrow XH(X)\subseteq XH(Y)$
\end{proposition}

\begin{proposition}
$XH$ satisfies: $\forall x,y\in E$, $X\subseteq E$,
\begin{center}
$y \in XH(X \bigcup \{x\})-XH(X)\Rightarrow x\in XH (X\bigcup \{y\})$
\end{center}
if and only if $\{N(x):x\in E\}$ forms a partition.
\end{proposition}

\begin{proof}
$\Rightarrow$: $\forall x,y\in E$, if $N(x)\bigcap N(y)\neq \emptyset$, then there exists $z\in N(x)\bigcap N(y)$.
Let $X=\emptyset$.
According to $(2)$ of Proposition 16, we know that if $y\in XH(\{x\})$ then $x\in XH(y)$, that is, if $x\in N(y)$ then $y\in N(x)$.
Since $z\in N(y)$, $N(z)\subseteq N(y)$.
According to the assumption, we also have $y\in N(z)$, that is, $N(y)\subseteq N(x)$.
Thus $N(x)=N(z)$.
Similarly, $z\in N(y)$, we have $N(z)=N(y)$.
Thus $N(x)=N(z)=N(y)$.
Hence, $\{N(x):x\in E\}$ forms a partition.

$\Leftarrow$: Since $XH(X \bigcup Y)=XH(X)\bigcup XH(Y)$ for all $X,Y\subseteq U$, $y\in XH(X\bigcup \{x\})\\-XH(X)=XH(X)\bigcup XH\{x\})-XH(X)=XH\{x\})-XH(X)$.
Now we prove $x\in XH(\{y\})$.
Since $y\in XH(\{x\})$, $x\in N(y)$.
Because the fact that $\{N(x):x\in E\}$ forms a partition, $x\in N(x)$ and $x\in N(y)$, we have $N(x)=N(y)$, thus $y\in N(x)$, that is, $x\in XH(\{y\})$.
Hence $x\in XH(\{y\})\subseteq XH(X\bigcup \{y\})$.
\end{proof}

\begin{theorem}
Let $\mathcal{C}$ be a covering. $\{N(x):x\in E\}$ induced by $\mathcal{C}$ forms a partition if and only if $XH$ is a closure operator of matroid.
\end{theorem}

\begin{proof}
It comes from (3), (5) and (6) of Proposition 16 and Proposition 17.
\end{proof}

\begin{definition}
Let $\mathcal{C}$ be a covering. If $\{N(x):x\in E\}$ induced by $\mathcal{C}$ forms a partition, then we define $\mathcal{I}^{''}=\{I\subseteq E: |I\bigcap N(x_{i})|\leq 1,\forall i\in \{1,2,\cdots,t\}\}$.
\end{definition}

Suppose $N(x_{1}),N(x_{2}),\cdots,N(x_{t})$ are all different neighborhoods on $E$.
As we know, if $N(x_{1}),N(x_{2}),\cdots,N(x_{t})$ forms a partition, then $XH$ is a closure operator of a matroid.
Moreover, $XH$ can determine a matroid, and its independent set is shown as follows:
\begin{center}
$\mathcal{I}_{XH}(\mathcal{C})$ $=\{I\subseteq E: \forall x\in I,x\notin XH(I-\{x\})\}$
\end{center}
Similar to the case of $SH$, we can obtain the following results.

\begin{proposition}
Let $\mathcal{C}$ be a covering. If $N(x_{1}),N(x_{2}),\cdots,N(x_{t})$ induced by $\mathcal{C}$ forms a partition, then $M(E,\mathcal{I}_{XH}(\mathcal{C})$ is a matroid and $\mathcal{I}_{XH}(\mathcal{C})=\mathcal{I}^{''}$.
\end{proposition}

\begin{proposition}
Let $\mathcal{C}$ be a covering. If $\{N(x_{1}),N(x_{2}),\cdots,N(x_{t})\}$ induced by $\mathcal{C}$ forms a partition and $M(E,\mathcal{I}_{XH}(\mathcal{C}))$ is the matriod induced by $XH$, then\\
(1) $X$ is a base of $M(E,\mathcal{I}_{XH}(\mathcal{C}))$ if and only if $|X\bigcap N(x_{i})|=1$ for all $1 \leq i\leq t$, and $M(E,\mathcal{I}_{XH}(\mathcal{C}))$ has $|N(x_{1})||N(x_{2}|\cdots |N(x_{s})|$ bases.\\
(2) For all $X\subseteq E$, $r_{XH}(X)=|\{N(x_{i}):N(x_{i})\bigcap X\neq \emptyset,1 \leq i\leq t\}|$.\\
(3) $X$ is a circuit of $M(E,\mathcal{I}_{XH}(\mathcal{C}))$ if and only if there exists $N(x_{i})$ such that $X\subseteq N(x_{i})$ and $|X|=2$.\\
(4) $X$ is a dependent set of $M(E,\mathcal{I}_{XH}(\mathcal{C}))$ if and only if there exists $N(x_{i})$ such that $|N(x_{i})\bigcap X|>1$.\\
(5) $\{N(x_{1}),N(x_{2}),\cdots,N(x_{t})\}$ is the set of all atoms of lattice $\mathcal{L}_{XH}(M(\mathcal{C}))$.\\
(6) There dose not exist $z\in E$ such that $x,y\in N(z)$ if and only if $XH(\{x,y\})$ covers $XH(\{x\})$ or $XH(\{y\})$.\\
(7) $\forall 1\leq i,j\leq t$, $(N(x_{i}),N(x_{j}))$ is a modular pair of lattice $\mathcal{L}_{XH}(M(\mathcal{C}))$.\\
(8) $\forall 1\leq i\leq t$, $N(x_{i})$ is a modular element of lattice $\mathcal{L}_{XH}(M(\mathcal{C}))$.
\end{proposition}

The proof of Proposition 18 and 19 is similar to that of Proposition 11 and Proposition 12 and 13, respectively.
So we omit the proof of them.

Similar to the case of $SH$, we also study what kind of covering can make $\{N(x):x\in E\}$ form a partition.
This paper establishes only two kinds of coverings.
As for others we can refer to \cite{YunGeBai11Axiomatization,HuXiaoZhang12Study}.

\begin{lemma}
Let $\mathcal{C}$ be a covering on $E$, $K$ reducible in $\mathcal{C}$.
$\forall x\in U$, $N(x)$ is the same in $\mathcal{C}$ as in $\mathcal{C}-\{K\}$.
\end{lemma}

\begin{proof}
$\forall x\in E$, $Md(x)$ is the same for covering $\mathcal{C}$ and covering $\mathcal{C}-\{K\}$, so $N(x)=\bigcap Md(x)$ is the same for the covering $\mathcal{C}$ and covering $\mathcal{C}-\{K\}$.
\end{proof}

\begin{proposition}
Let $\mathcal{C}$ be a covering.
If $reduct(\mathcal{C})$ is a partition, then $\{N(x):x\in E\}$ induced by $\mathcal{C}$ is also a partition.
\end{proposition}

\begin{proof}
Since $reduct(\mathcal{C})$ is a partition of $E$, $\{N(x):x\in E\}$ induced by $reduct(\mathcal{C})$ forms a partition.
Suppose $\{K_{1},K_{2},\cdots, K_{s}\}$ is the set all reducible elements of $\mathcal{C}$.
According to Lemma 12, we know that $\forall x\in E$, $N(x)$ is the same in $reduct(\mathcal{C})$ as in $reduct(\mathcal{C})\bigcup \{K_{1}\}$, thus $\{N(x):x\in E\}$ induced by $reduct(\mathcal{C})\bigcup\{K_{1}\}$ forms a partition.
And the rest may be deduced by analogy, then we can obtain $\forall x\in E$, $N(x)$ is the same in $reduct(\mathcal{C})$ as in $\mathcal{C}$, thus $\{N(x):x\in E\}$ induced by $\mathcal{C}$ forms a partition.
\end{proof}

\begin{proposition}
Let $\mathcal{C}$ be a covering on $E$. If $\mathcal{C}$ satisfies the $(EQU)$ condition:
$\forall K\in \mathcal{C}$, $\forall x,y \in K$, the number of blocks which contain $x$ is equal to that of blocks which contain $y$, then $\{N(x):x\in E\}$ induced by $\mathcal{C}$ forms a partition.
\end{proposition}

\begin{proof}
$\forall x, y\in E$, if $N(x)\bigcap N(y)\neq \emptyset$, then there exists $z\in E$ such that $z\in N(x)$ and $z\in N(y)$, that is, the blocks which contain $x$ also contain $z$ and the blocks which contains $y$ also contain $z$.
Hence, there exist $K_{i}, K_{i}^{'}$ such that $x,z\in K_{i}$ and $y,z\in K_{i}^{'}$.
Without loss of generality, we suppose $\{K_{1}, K_{2}, \cdots, K_{s}\}$ is the set of all blocks which contain $x$ and $\{K_{1}^{'},K_{2}^{'},\cdots, K_{t}^{'}\}$ is the set of all blocks which contain $y$.
Since the number of blocks which contain $z$ is equal to that of blocks which contain $x$, $\{K_{1}, K_{2}, \cdots, K_{s}\}$ is the set of all blocks which contain $z$, thus $\{K_{1}^{'},K_{2}^{'},\cdots, K_{t}^{'}\}\subseteq \{K_{1}, K_{2}, \cdots, K_{s}\}$.
Hence, $N(x)\subseteq N(y)$.
Similarly, we can prove $N(y)\subseteq N(x)$.
Hence, $N(x)=N(y)$, that is, $\{N(x):x\in E\}$ forms a partition.
\end{proof}

Qin et al. first defined the seventh type of covering-based upper approximation in \cite{QinGaoPei07OnCovering}.
They also discussed the relationship between it and other types of approximations.
Proposition 22 below gives some properties of the covering upper approximation operator $VH$.

\begin{proposition}
Let $\mathcal{C}$ be a covering of $E$. $VH$ has the following properties:\\
(1) $VH(\emptyset)=\emptyset$\\
(2) $X\subseteq VH(X)$ for all $X\subseteq E$\\
(3) $VH(X\bigcup Y)=VH(X)\bigcup VH(Y)$\\
(4) $x\in VH(\{y\})\Leftrightarrow y\in VH(\{x\})$\\
(5) $X\subseteq Y\subseteq E\Rightarrow VH(X)\subseteq VH(Y)$\\
(6) $\forall x,y \in E, y \in VH(X\bigcup \{x\})-VH(X)$, then $x\in VH(X\bigcup \{y\})$
\end{proposition}

\begin{proof}
(1)-(3) were shown in \cite{QinGaoPei07OnCovering}, so we prove only (4), (5) and (6).\\
(4) $y\in VH(\{x\})\Leftrightarrow$ there exists $z\in U $ such that $x,y \in N(z)\Leftrightarrow x\in VH(\{y\})$.\\
(5) $\forall x\in VH(X)$, there exists $y \in E$ such that $x\in N(y)$ and $N(y)\bigcap X\neq \emptyset$.
Since $X\subseteq Y$, $N(y)\bigcap Y\neq \emptyset$. Hence, $x\in VH(Y)$, that is, $VH(X)\subseteq VH(Y)$.\\
(6) According to (3), we have $VH(X\bigcup \{x\})-VH(X)=VH(X)\bigcup VH(\{x\})-VH(X)=VH(\{x\}-VH(X)$.
If $y\in VH(\{x\}-VH(X)$, then $y\in VH(\{x\})$.
Base on (4), we know $x\in VH(\{y\})\subseteq VH(X)\bigcup VH(\{y\})=VH(X\bigcup \{y\})$.
Hence, (5) holds.
\end{proof}

\begin{proposition}
$VH(VH(X))=VH(X)$ if and only if $\{VH(\{x\}):x\in E\}$ forms a partition of $E$.
\end{proposition}

\begin{proof}
"$\Leftarrow$": According to (2) and (5) of Proposition 22, we know $VH(X)\subseteq VH(VH\\(X))$.
For all $x\in VH(VH(X))$, there exists $y\in VH(X)$ such that $x\in VH(\{y\})$.
Since $y\in VH(X)$, there exists $z\in E$ such that $y\in VH(\{z\})$.
Based on $y\in VH(\{y\})$ and $\{VH(\{x\}):x\in E\}$ forms a partition, we have $VH(\{y\})=VH(\{z\})$.
Thus $x\in VH(\{z\})$ because $x\in VH(\{y\})$.
Hence, $x\in VH(X)$, that is, $VH(VH(X))\\\subseteq VH(X)$.

"$\Rightarrow$": In order to prove $\{VH(\{x\}):x\in E\}$ forms a partition, we need to prove that for all $x,y\in E$, if $VH(\{x\})\bigcap VH(\{y\})\neq \emptyset$, then $VH(\{x\})=VH(\{y\})$.
If $VH(\{x\})\bigcap VH(\{y\})\neq \emptyset$, then there exists $z\in VH(\{x\})\bigcap VH(\{y\})$.
For $VH(VH$ $(\{x\}))=\bigcup \{VH(\{u\}):u\in VH(\{x\})\}$ and $z\in VH(\{x\})$, $VH(\{z\})\subseteq VH(VH\\(\{x\}))=VH(\{x\})$.
Based on (4) of Proposition 15 and $z\in VH(\{x\})$, we have $x\in VH(\{z\})$.
Thus it is easy to obtain $VH(\{x\})\subseteq VH(VH(\{z\}))=VH(\{z\})$.
Hence, $VH(\{x\})=VH(\{z\})$.
Similarly, we can obtain $VH(\{y\})=VH(\{x\})$, thus $VH(\{x\})=VH(\{z\})=VH(\{y\})$.
\end{proof}

\begin{theorem}
$VH$ is a closure operator of matroid $M$ if and only if $\{VH(\{x\}):x\in E\}$ forms a partition $E$.
\end{theorem}

\begin{proof}
It comes from (2), (5) and (6) of Proposition 2, 22 and 23.
\end{proof}

\begin{definition}
Let $\mathcal{C}$ be a covering. If $\{VH(\{x\}):x\in E\}$ forms a partition, then we define $\mathcal{I}^{'''}=\{I\subseteq E:|I\bigcap VH$ $(\{x_{i}\})|\leq 1,\forall i\in \{1,2,\cdots, t\}\}$.
\end{definition}

If $\{VH(\{x\}):x\in E\}=VH(\{x_{1}\}),VH(\{x_{2}\}),\cdots,$ $VH$ $(\{x_{r}\})$ forms a partition of $E$, then $VH$ is a closure operator of a matroid, and the matroid's independent set is shown as follows:
\begin{center}
$\mathcal{I}_{VH}(\mathcal{C})=\{I\subseteq E:\forall x\in I, x\notin VH(I-\{x\})\}$
\end{center}
Similar to Proposition 11, 12 and 13, we can also obtain the following two results.

\begin{proposition}
If $VH(\{x_{1}\}),VH(\{x_{2}\}),\cdots,VH(\{x_{r}\})$ forms a partition on $E$, then $M(E,\mathcal{I}_{VH}(\mathcal{C}))$ is a matroid and $\mathcal{I}_{VH}(\mathcal{C})=\mathcal{I}^{'''}$.
\end{proposition}

\begin{proposition}
If $\{VH(\{x_{1}\}),VH(\{x_{2}\}),\cdots,VH(\{x_{r}\})\}$ forms a partition on $E$ and $M(E,\mathcal{I}_{VH}(\mathcal{C}))$ is the matriod induced by $VH$, then\\
(1) For all $X\subseteq E$, $r_{VH}(X)=|\{VH(\{x_{i}\}):VH(\{x_{i}\})\bigcap X\neq \emptyset,i=1,2,\cdots,r\}|$.\\
(2) $X$ is a base of $M(E,\mathcal{I}_{VH}(\mathcal{C}))$ if and only if $|X\bigcap VH(\{x_{i}\})|=1$ for all $1\leq i\leq r$, and $M(E,\mathcal{I}_{VH}(\mathcal{C}))$ has $|VH(\{x_{1}\})||VH(\{x_{2}\})|\cdots |VH(\{x_{r}\})|$ bases.\\
(3) $X$ is a circuit of $M(E,\mathcal{I}_{VH}(\mathcal{C}))$ if and only if there exists $VH(\{x_{i}\})$ such that $X\subseteq VH(\{x_{i}\})$ and $|X|=2$.\\
(4) $X$ is a dependent set of $M(E,\mathcal{I}_{VH}(\mathcal{C}))$ if and only if there exists $VH(\{x_{i}\})$ such that $|VH(\{x_{i}\})\bigcap X|>1$.\\
(5) $\{VH(\{x_{1}\}),VH(\{x_{2}\}),\cdots,VH(\{x_{r}\})\}$ are all atoms of $\mathcal{L}_{VH}(M(\mathcal{C}))$.\\
(6) There dose not exists $z\in E$ such that $x,y\in \{VH(\{z\}$ if and only if $VH(\{x,y\})$ covers $VH(\{x\})$ or $VH(\{y\})$.\\
(7) $\forall 1\leq i,j\leq r, (\{VH(\{x_{i}\}),\{VH(\{x_{j}\}))$ is a modular pair of lattice $\mathcal{L}_{VH}(M(\mathcal{C}))$.\\
(8) $\forall 1\leq i\leq r, VH(x_{i})$ is a modular element of lattice $\mathcal{L}_{VH}(M(\mathcal{C}))$.
\end{proposition}

The proof of Proposition 24, 25 is similar to that of Proposition 11,12 and 13, respectively.
We omit it here.

As we know, the seventh type of upper approximation is defined by neighborhood, and the sixth and the seventh types of upper approximations are
equivalent when $\{N(x):x\in E\}$ forms a partition.
Thus the covering which make the sixth type of upper approximation be a closure of a
matroid is the covering which make the seventh type of upper approximation be a closure operator of a matroid, hence we omit the discussion about what kind of covering can make the seventh type of upper approximation be a closure operator of a matroid.

\section{Relationships among four geometric lattice structures of covering-based rough sets}
\label{S:Therelationshipamongfourtypesofmatroidalstructuresandclosedsetlatticestructures}

In Section~\ref{S:Thepropertiesoftheclosedsetlatticeofmatroidinducedbycovering}, the properties of the geometric lattice have been studied by
matroid $M(E, \mathcal{I}(\mathcal{C}))$, and we also have studied the properties of matroids $M(E,\mathcal{I}_{SH}(\mathcal{C}))$, $M(E,\mathcal{I}_{XH}(\mathcal{C}))$.
Section~\ref{S:Matroidalstructuresbasedonthreetypesofupperapproximations} presents sufficient and necessary conditions for three types of covering upper approximation operators to be closure operators of matroids.
Moreover, we exhibit three types of matroids through closure axioms, and then obtain three geometric lattice structures of covering-based rough sets
In this section, we compare above four types of geometric lattices through corresponding matroids.
We also discuss the reducible element and the immured element's influence on the relationship among this four types of matroidal structures and geometric lattice structures.

The following proposition shows the relationship between $\mathcal{I}_{SH}(\mathcal{C})$ and $\mathcal{I}(\mathcal{C})$, and the relationship between $\mathcal{L}_{SH}(M(\mathcal{C}))$ and $\mathcal{L}(M(\mathcal{C}))$.

\begin{proposition}
Let $\mathcal{C}$ be a covering. if $SH$ induced by $\mathcal{C}$ is a closure operator, then $\mathcal{I}_{SH}(\mathcal{C})\subseteq \mathcal{I}(\mathcal{C})$ and $\mathcal{L}_{SH}(M(\mathcal{C}))\subseteq \mathcal{L}(M(\mathcal{C}))$.
\end{proposition}

\begin{proof}
Since $SH$ induced by $\mathcal{C}$ is a closure operator, $\{I(x_{1}),I(x_{2}),\cdots,\\I(x_{s})\}$ forms a partition.
$\forall I\in \mathcal{I}_{SH}(\mathcal{C})$, suppose $I=\{i_{1},i_{2},\cdots,i_{\alpha}\}(\alpha \leq s)$ such that $i_{1} \in I(x_{i_{1}}), i_{2} \in I(x_{i_{2}}),\cdots,i_{\alpha} \in I(x_{i_{\alpha}})$ and $\{I(x_{i_{1}}),I(x_{i_{2}}),\cdots,I(x_{i_{\alpha}})\}\subseteq \{I(x_{1}),I(x_{2}),\cdots,I(x_{s})\}$.
According to the definition of $I(x)=\bigcup_{x\in K}K$, there exists $\{K_{i_{1}}, K_{i_{2}},\cdots K_{i_{\alpha}}\}\subseteq \mathcal{C}$ such that $i_{1}\in K_{i_{1}}, i_{2}\in K_{i_{2}}, \cdots, i_{\alpha}\in K_{i_{\alpha}}$.
Since $\{I(x_{1}),I(x_{2}),\cdots,I(x_{s})\}$ forms a partition, thus $\{K_{i_{1}}, K_{i_{2}},\cdots K_{i_{\alpha}}\}$ are different bloc-ks.
According to the definition of transversal matroid, we have $I\in \mathcal{I}$.
Hence, $\mathcal{I}_{SH}(\mathcal{C})\subseteq \mathcal{I}(\mathcal{C})$.

$\forall X\in \mathcal{L}_{SH}(M(\mathcal{C}))$, $X=SH(X)=\bigcup_{x\in X}I(x)$.
Now we prove $X\in \mathcal{L}(M(\mathcal{C}))$, that is, $cl_{M(\mathcal{C})}(X)=X=\{x|r_{M(\mathcal{C})}(X)=r_{M(\mathcal{C})}(X\bigcup \{x\})\}$. Since $X\subseteq cl_{M(\mathcal{C})}(X)$, if $cl_{M(\mathcal{C})}(X)\neq X$, then $cl_{M(\mathcal{C})}(X)\nsubseteq X$, that is, there exists $y\in E$ such that $r_{M(\mathcal{C})}(X)=r_{M(\mathcal{C})}(X\bigcup \{y\})$ and $y\notin X$.
Suppose $T=\{t_{1}, t_{2}, \cdots, t_{t}\}(t\leq s)$ is a maximal independent set included in $X$, then $\{t_{1}, t_{2}, \cdots, t_{t}\}\subseteq X=\bigcup_{x\in X}I(x)$ and there exist different $K_{1},K_{2},\cdots, K_{t}$ such that $\forall i\in \{1,2,\cdots,t\}, t_{i}\in K_{i}$.
Since $y\notin X$, $\forall x\in X$, $I(x)\bigcap I(y)=\emptyset$.
Based on $\{I(x_{1}),I(x_{2}),\cdots,I(x_{s})\}$ forms a partition, there exists $K\subseteq I(y)$ such that $K_{1},K_{2},\cdots, K_{t}, K$ are different blocks and $y\in K$, thus $T\bigcup \{y\}$ is a maximal independent set included in $X\bigcup \{y\}$.
Hence, we have $r_{M(\mathcal{C})}(X\bigcup \{y\})=r_{M(\mathcal{C})}(X)+1$ which contradicts $r_{M(\mathcal{C})}(X)=r_{M(\mathcal{C})}(X\bigcup \{y\})$.
Thus we can obtain $cl_{M(\mathcal{C})}(X)=X$.
\end{proof}

The following proposition illustrates that in what condition the indiscernible neighborhoods are included in the closed-set lattice induced by $\mathcal{C}$.

\begin{proposition}
Let $\mathcal{C}$ be a covering. If $SH$ induced by $\mathcal{C}$ is a closure operator, then $\forall x\in E, I(x)\in \mathcal{L}(\mathcal{M}(\mathcal{C}))$.
\end{proposition}

\begin{proof}
Since $SH$ induced by $\mathcal{C}$ is a closure operator, $\{I(x):x\in E\}$ forms a partition.
Thus, for all $x\in E$, $SH(I(x))=\bigcup_{y\in I(x)}I(y)=I(x)$, that is, $I(x)\in \mathcal{L}_{SH}(M(\mathcal{C}))$.
According to proposition 26, $I(x)\in \mathcal{L}(\mathcal{M}_{\mathcal{C}})$.
\end{proof}

\begin{example}
Let $\mathcal{C}=\{\{1,2\},\{1,3\},\{2,3\},\{4,5\}\}$.
$I(1)=I(2)=I(3)=\{1,2,3\}$, $I(4)=I(5)=\{4,5\}$.
$\mathcal{L}_{SH}(M(\mathcal{C}))=\{\emptyset, \{1,2,3\},\{4,5\},\{1,2,3,4,5\}\}$.
$\mathcal{L}(M(\mathcal{C}))\\=\{\emptyset, \{1\},\{2\},\{3\},\{4,5\},\{1,2\},\{1,3\},\{2,3\},\{1,4,5\},\{2,4,5\},\{3,4,5\},\{1,2,3\},\\\{1,2,4,5\},\{1,3,4,5\},\{2,3,4,5\},\{1,2,3,4,5\}\}$, the structures of $\mathcal{L}(\mathcal{M}_{SH})$ and $\mathcal{L}(\mathcal{M}_{\mathcal{C}})$ are showed in figure 1.
\begin{figure*}[ht]

   \begin{center}
   \includegraphics[width=5in]{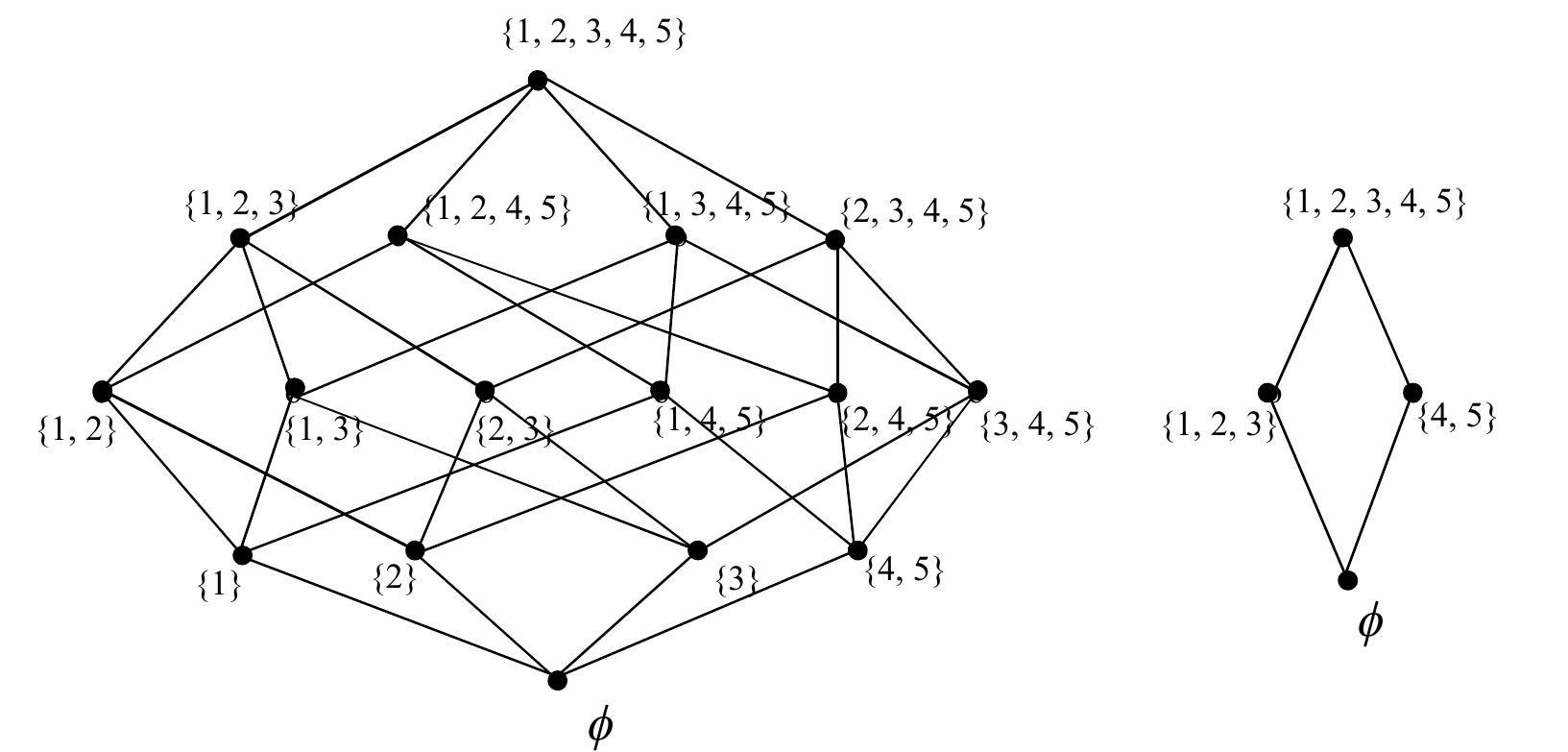}
   \caption{The lattice of $\mathcal{L}(M(\mathcal{C}))$(resp.$\mathcal{L}_{XH}(M(\mathcal{C}))$) and $\mathcal{L}_{SH}(M(\mathcal{C}))$}

    \end{center}

\end{figure*}
\end{example}

\begin{remark}
Let $\mathcal{C}$ is a covering and $XH$ induced by $\mathcal{C}$ a closure operator. However, it has no relationship between $\mathcal{I}_{XH}(\mathcal{C})$ and $\mathcal{I}(\mathcal{C})$, and has no relationship between $\mathcal{L}_{XH}(M(\mathcal{C}))$ and $\mathcal{L}(M(\mathcal{C}))$.
Similarly, those conclusions also hold for $VH$.
\end{remark}

The following example illustrates the above statements.

\begin{example}
$\mathcal{C}=\{K_{1},K_{2},K_{3},K_{4}\}$ is a covering of $E$, Where $K_{1}=\{a,b,i\}, K_{2}=\{a,b,c,d,e,f\}, K_{3}=\{f,g,h\}, K_{4}=\{c,d,e,g,h,i\}$.
Then $VH(\{a\})=VH(\{b\})=N(a)=N(b)=\{a,b\},VH(\{c\})=VH(\{d\})=VH(\{e\})=N(c)=N(d)=N(e)=\{c,d,e\}, VH(\{f\})=N(f)=\{f\}, VH(\{g\})=VH(\{h\})=N(g)=N(h)=\{g,h\}, VH(\{i\})=N(i)=\{i\}$.
Let $T=\{a,c,f,g,i\}$.
It is clear that $T\in \mathcal{I}_{XH}(\mathcal{C})=\mathcal{I}_{VH}(\mathcal{C})$, but $T\notin \mathcal{I}(\mathcal{C})$, thus $\mathcal{I}_{XH}(\mathcal{C})\nsubseteq \mathcal{I}(\mathcal{C})$ and $\mathcal{I}_{VH}(\mathcal{C})\nsubseteq \mathcal{I}(\mathcal{C})$.
Let $T^{'}=\{a,c,d\}$.
It is clear that $T^{'}\in \mathcal{I}(\mathcal{C})$, but $T^{'}\notin \mathcal{I}_{XH}\mathcal{I}_{VH}$ for $|T^{'}\bigcap N(a)|=2$ and $|T^{'}\bigcap VH(\{a\})|=2$, thus $\mathcal{I}(\mathcal{C})\nsubseteq \mathcal{I}_{XH}\mathcal{C}$ and $\mathcal{I}(\mathcal{C})\nsubseteq \mathcal{I}_{VH}(\mathcal{C})$.
Let $X=\{a,b,i\}$. $X\in \mathcal{L}_{XH}(M(\mathcal{C}))$ for $XH(X)=X$.
However, $X\notin \mathcal{L}(M(\mathcal{C}))$ for $cl_{M(\mathcal{C})}(X)=\{a,b,c,d,e,i\}$.
Let $X=\{a\}$. $X\in \mathcal{L}(M(\mathcal{C}))$ for $cl_{M(\mathcal{C})}(X)=X$.
However, $X\notin \mathcal{L}_{XH}(M(\mathcal{C}))$ for $XH(X)=\{a,b\}\neq X$.
\end{example}

\begin{proposition}
Let $\mathcal {C}$ be a covering. If $XH$ induced by $\mathcal{C}$ is a closure operator, then $VH$ induced by $\mathcal{C}$ is also a closure operator.
Moreover, $\mathcal{I}_{XH}(\mathcal{C})=\mathcal{I}_{VH}(\mathcal{C})$ and $\mathcal{L}_{XH}(M(\mathcal{C}))=\mathcal{L}_{VH}(M(\mathcal{C}))$.
\end{proposition}

\begin{proof}
Since $XH$ induced by $\mathcal{C}$ is a closure operator, $\{N(x):x\in E\}$ forms a partition.
$\forall X\subseteq E$, $XH(X)=\{x:N(x)\bigcap X\neq \emptyset\}$ and $VH(X)=\bigcup\{N(x):N(x)\bigcap X\neq \emptyset\}$.
It is clear that $XH(X)\subseteq VH(X)$.
$\forall x\in VH(X)$, there exists $y\in E$ such that $x\in N(y)$ and $N(y)\bigcap X\neq \emptyset$.
Based on $x\in N(x)$ and $\{N(x):x\in E\}$ forms a partition, $N(x)=N(y)$ and $N(y)\bigcap X\neq \emptyset$, thus $x\in XH(X)$, that is, $VH(X)\subseteq XH(X)$.
Hence, $VH(X)=XH(X)$.
According to the definition of $\mathcal{L}_{XH}(M(\mathcal{C}))$ and $\mathcal{L}_{VH}(M(\mathcal{C}))$, we can obtain $\mathcal{L}_{XH}(M(\mathcal{C}))=\mathcal{L}_{VH}(M(\mathcal{C}))$.
$\forall X\subseteq E$, the equality $XH(X)=VH(X)$ holds, so $\forall x\in E$, $VH(\{x\})=N(x)$.
According to the definition of $\mathcal{I}_{XH}(\mathcal{C})$ and $\mathcal{I}_{VH}(\mathcal{C})$, we have $\mathcal{I}_{XH}(\mathcal{C})=\mathcal{I}_{VH}(\mathcal{C})$.
\end{proof}

\begin{theorem}
Let $\mathcal{C}$ be a covering. If $XH$ and $SH$ induced by $\mathcal{C}$ are closure operators, then
$\mathcal{I}_{SH}(\mathcal{C})\subseteq \mathcal{I}_{XH}(\mathcal{C})=\mathcal{I}_{VH}(\mathcal{C})$ and $\mathcal{L}_{SH}(M(\mathcal{C}))\subseteq \mathcal{L}_{XH}(M(\mathcal{C}))= \mathcal{L}_{VH}(M(\mathcal{C}))$.
\end{theorem}

\begin{proof}
If $XH$ and $SH$ induced by $\mathcal{C}$ are closure operators, then $\{N(x):x\in E\}$ and $\{I(x):x\in E\}$ forms a partition of $E$, respectively.
$\forall x\in X$, $N(x)=\bigcap_{x\in K}K\subseteq \bigcup_{i\in K}K=I(x)$, thus $\{N(x):x\in E\}$ is finer than $\{I(x):x\in E\}$.
Based on this, we can obtain $\mathcal{I}_{SH}(\mathcal{C})\subseteq \mathcal{I}_{XH}(\mathcal{C})$.
According to Proposition 28, we have  $\mathcal{I}_{SH}(\mathcal{C})\subseteq \mathcal{I}_{XH}(\mathcal{C})=\mathcal{I}_{VH}(\mathcal{C})$.

For all $X\in \mathcal{L}_{SH}(M(\mathcal{C}))$, $X=SH$$(X)=\bigcup_{x\in X}I(x)$.
Since $x\in N(x)\subseteq I(x)$, $X=\bigcup_{x\in X}\{x\}\subseteq \bigcup_{x\in X}N(x)\subseteq \bigcup_{x\in X}I(x)=X$, thus, $X=\bigcup_{x\in X}N(x)$, that is, $X\in \mathcal{L}_{XH}(M(\mathcal{C}))$.
Hence, $\mathcal{L}_{SH}(M(\mathcal{C}))\subseteq \mathcal{L}_{XH}(M(\mathcal{C}))$.
By the fact that $\{N(x):x\in E\}$ forms a partition and Proposition 28, $\mathcal{I}_{XH}(\mathcal{C})=\mathcal{I}_{VH}(\mathcal{C})$.
Thus $\mathcal{L}_{SH}(M(\mathcal{C}))\subseteq \mathcal{L}_{XH}(M(\mathcal{C}))= \mathcal{L}_{VH}(M(\mathcal{C}))$.
\end{proof}

\begin{example}
From Example 2, we know $N(1)=\{1\}, N(2)=\{2\}, N(3)=\{3\}, N(4)\\=N(5)=\{4,5\}$ and $I(1)=I(2)=I(3)=\{1,2,3\}$, $I(4)=I(5)=\{4,5\}$. $\mathcal{L}_{SH}(M(\mathcal{C}))=\{\emptyset, \{1,2,3\},\{4,5\},\{1,2,3,4,5\}\}$ and $\mathcal{L}_{XH}(M(\mathcal{C}))=\{\emptyset, \{1\},\\\{2\},$ $\{3\},\{4,5\},\{1,2\},\{1,3\},\{2,3\},\{1,4,5\},\{2,4,5\},\{3,4,5\},\{1,2,3\},\{1,2,4,5\},\\\{1,3,4,5\},\{2,3,4,5\},\{1,2,3,4,5\}\}$.
The structures of them are showed in figure 1.
\end{example}

When a covering degenerates into a partition, we can obtain the following result.

\begin{theorem}
If $\mathcal{C}$ is a partition, then $\mathcal{I}_{SH}(\mathcal{C})=\mathcal{I}_{XH}(\mathcal{C})=\mathcal{I}_{VH}(\mathcal{C})=\mathcal{I}(\mathcal{C})$ and $\mathcal{L}_{SH}(M(\mathcal{C}))=\mathcal{L}_{XH}(M(\mathcal{C}))=\mathcal{L}_{VH}(M(\mathcal{C}))=\mathcal{L}(M(\mathcal{C}))$.
\end{theorem}

\begin{proof}
Since $\mathcal{C}$ is a partition, $\forall x\in E$, $I(x)=N(x)=VH(\{x\})=K$ where $x\in K$, and $\forall X\subseteq E$, $SH(X)=XH(X)=VH(X)$.
Thus $\mathcal{I}_{SH}(\mathcal{C})=\mathcal{I}_{XH}(\mathcal{C})=\mathcal{I}_{VH}(\mathcal{C})=\mathcal{I}(\mathcal{C})$ and $\mathcal{L}_{SH}(M(\mathcal{C}))=\mathcal{L}_{XH}(M(\mathcal{C}))=\mathcal{L}_{VH}(M(\mathcal{C}))=\mathcal{L}(M(\mathcal{C}))$.
\end{proof}

Next, we discuss the reducible element and immured element's influence on the independent set and the closed-set lattice.

\begin{theorem}
Let $\mathcal{F}$ be a family subset on $E$, $K\in \mathcal{F}$.
$\mathcal{I}(\mathcal{F}-\{K\})\subseteq \mathcal{I}(\mathcal{F})$.
\end{theorem}

\begin{proof}
For all $I\in \mathcal{I}(\mathcal{F}-\{K\})$, we may as well suppose $I=\{i_{1}, i_{2}, \cdots, i_{t}\}, i_{1},i_{2},\cdots\\ i_{t}\in E$.
According to the definition of transversal matroid, there exist different blocks $K_{1},K_{2},\cdots, \\K_{t}\in \mathcal{F}$ satisfy $K_{i}\neq K$ and $i_{j}\in K_{j}$ for all $1\leq j\leq t$.
Thus $I\in \mathcal{I}(\mathcal{F})$.
\end{proof}

The following example illustrates $\mathcal{I}(\mathcal{F})\nsubseteq \mathcal{I}(\mathcal{F}-\{K\})$.

\begin{example}
Let $\mathcal{F}=\{K_{1},K_{2},K_{3}\}$ be a family subset of $E=\{1,2,3,4\}$, where $K_{1}=\{1,2\}$, $K_{2}=\{1,3\}$, $K_{3}=\{3\}$. $\mathcal{I}(\mathcal{F})=P(E)$, $\mathcal{I}(\mathcal{F}-\{K_{3}\})=\{\emptyset,\{1\},\{2\},\\\{3\},\{1,3\},\{1,2\} \{2,3\}\}$.
Hence, $\mathcal{I}(\mathcal{F})\nsubseteq \mathcal{I}(\mathcal{F}-\{K\})$.
\end{example}

\begin{corollary}
Let $\mathcal{C}$ be a covering on $E$, $K\in \mathcal{C}$.
If $K$ is reducible, then $\mathcal{I}(\mathcal{C}-\{K\})\subseteq \mathcal{I}(\mathcal{C})$.
\end{corollary}

\begin{corollary}
Let $\mathcal{C}$ be a covering.
$\mathcal{I}(reduct(\mathcal{C}))\subseteq \mathcal{I}(\mathcal{C})$.
\end{corollary}

\begin{corollary}
Let $\mathcal{C}$ be a covering, $K\in \mathcal{C}$.
If $K$ is an immured element, then $\mathcal{I}(\mathcal{C}-\{K\})\subseteq \mathcal{I}(\mathcal{C})$.
\end{corollary}

\begin{corollary}
Let $\mathcal{C}$ be a covering.
$\mathcal{I}(exclusion(\mathcal{C}))\subseteq \mathcal{I}(\mathcal{C})$
\end{corollary}

\begin{theorem}
Let $\mathcal{F}$ be a family subset on $E$, $\forall K\in \mathcal{F}$.
$\mathcal{L}(M(\mathcal{F}-\{K\}))\subseteq \mathcal{L}(M(\mathcal{F}))$.
\end{theorem}

\begin{proof}
First, we prove the result $cl_{\mathcal{F}}(\{x\})\subseteq cl_{\mathcal{F}-\{K\}}(\{x\})$.
For all $y\notin cl_{\mathcal{F}-\{K\}}$, $\{x,y\}\in \mathcal{I}_{\mathcal{F}-\{K\}}$.
Since $\{x,y\}\in \mathcal{I}_{\mathcal{F}}$, thus $y\notin cl_{\mathcal{F}}(\{x\})$.
That implies that $cl_{\mathcal{F}}(\{x\})\subseteq cl_{\mathcal{F}-\{K\}}(\{x\})$.

Second, we prove that any atom of $\mathcal{L}(M(\mathcal{F}-\{K\}))$ is a closed set of $\mathcal{L}(M(\mathcal{F}))$, that is, $cl_{\mathcal{F}}(cl_{\mathcal{F}-\{K\}}(\{x\}))=cl_{\mathcal{F}-\{K\}}(\{x\})$.
Since $cl_{\mathcal{F}-\{K\}}(\{x\})\subseteq cl_{\mathcal{F}}(cl_{\mathcal{F}-\{K\}}(\\\{x\}))\subseteq cl_{\mathcal{F}-\{K\}}(cl_{\mathcal{F}-\{K\}}(\{x\}))=cl_{\mathcal{F}-\{K\}}(\{x\})$,
$cl_{\mathcal{F}}(cl_{\mathcal{F}-\{K\}}(\{x\}))=cl_{\mathcal{F}-\{K\}}\\(\{x\})$.

Third, we prove that $\mathcal{L}(M(\mathcal{F}-\{K\}))\subseteq \mathcal{L}(M(\mathcal{F}))$.
$\forall X\subseteq \mathcal{L}(M(\mathcal{F}-\{K\}))$,  $X=\bigvee_{i=1}^{m}cl_{\mathcal{F}-\{K\}}(\{x\})=\bigvee_{i=1}^{m}\bigvee_{j=1}^{s_{i}}cl_{\mathcal{F}}\\(\{x_{j}\})$ because $\mathcal{L}(M(\mathcal{F}-\{K\}))$ is a atomic lattice, thus $X\in \mathcal{L}(M(\mathcal{F}))$.
\end{proof}

\begin{example}
Based on Example 5, we have $\mathcal{L}(M(\mathcal{F}))=P(E)$ and $\mathcal{L}(M(\mathcal{F}-\{K_{3}\}))=\{\emptyset, \{1\},\\\{2\},\{3\},\{1,2,3\}\}$.
It is clear that $\mathcal{L}(M(\mathcal{F}))\nsubseteq \mathcal{L}(M(\mathcal{F}-\{K_{3}\}))$.
\end{example}

\begin{corollary}
Let $\mathcal{C}$ be a covering on $E$, $K\in \mathcal{C}$.
If $K$ is reducible, then $\mathcal{L}(M(\mathcal{C}-\{K\}))\subseteq \mathcal{L}(M(\mathcal{C}))$.
\end{corollary}

\begin{corollary}
Let $\mathcal{C}$ be a covering on $E$.
$\mathcal{L}(M(reduct(\mathcal{C})))\subseteq \mathcal{L}(M(\mathcal{C}))$.
\end{corollary}

\begin{corollary}
Let $\mathcal{C}$ be a covering on $E$, $K\in \mathcal{C}$.
If $K$ is an immured element, then $\mathcal{L}(M(\mathcal{C}-\{K\}))\subseteq \mathcal{L}(M(\mathcal{C}))$.
\end{corollary}

\begin{corollary}
Let $\mathcal{C}$ be a covering on $E$.
$\mathcal{L}(M(exclusion(\mathcal{C})))\subseteq \mathcal{L}(M(\mathcal{C}))$.
\end{corollary}

\begin{remark}
Let $\mathcal{C}$ is a covering and $SH$ induced by $\mathcal{C}$ a closure operator.
If a reducible element $K$ is removed from the covering $\mathcal{C}$, then $\mathcal{C}-\{K\}$ may not still be a covering such that $SH$ is a closure operator.
Hence, it is difficult to discuss the relationship between $\mathcal{I}_{SH}(\mathcal{C})$ and $\mathcal{I}_{SH}(\mathcal{C}-\{K\})$.
\end{remark}

\begin{example}
Let $E=\{1,2,3\}$ and $\mathcal{C}=\{K_{1}, K_{2}, K_{3}\}$ where $K_{1}=\{1,2\},K_{2}=\{1,3\},K_{3}=\{1,2,3\}$.
$I(1)=I(2)=I(3)=\{1,2,3\}$, thus $\{I(x):x\in E\}$ forms a partition.
Hence, $SH$ is a closure operator induced by $\mathcal{C}$.
It is clear that $K_{3}$ is a reducible element, $\mathcal{C}-\{K_{3}\}=\{K_{1},K_{2}\}$.
Then the indiscernible neighborhoods induced by $\mathcal{C}-\{K_{3}\}$ are $I(1)=\{1,2,3\},I(2)=\{1,2\},I(3)=\{1,3\}$.
we find that $\{I(x):x\in E\}$ can not form a partition.
Hence, $SH$ is not a closure operator induced by $\mathcal{C}-\{K_{3}\}$.
\end{example}

\begin{theorem}
Let $\mathcal{C}$ be a covering on $E$ and $K$ an immured element. If $SH$ induced by $\mathcal{C}$ is a closure operator, then $SH$ induced by $\mathcal{C}-\{K\}$ is also a closure operator.
Moreover, $\mathcal{I}_{SH}(\mathcal{C})=\mathcal{I}_{SH}(\mathcal{C}-\{K\})$ and $\mathcal{L}_{SH}(M(\mathcal{C}))=\mathcal{L}_{SH}(M(\mathcal{C}-\{K\}))$.
\end{theorem}

\begin{corollary}
Let $\mathcal{C}$ be a covering on $E$. If $SH$ induced by $\mathcal{C}$ is a closure operator, then $SH$ induced by
$exclusion(\mathcal{C})$ is also a closure operator.
Moreover, $\mathcal{I}_{SH}(\mathcal{C})=\mathcal{I}_{SH}(exclusion(\mathcal{C}))$ and $\mathcal{L}_{SH}(M(\mathcal{C}))=\mathcal{L}_{SH}(M(exclusion(\mathcal{C})))$.
\end{corollary}

\begin{remark}
Let $\mathcal{C}$ be a covering and $XH$ induced by $\mathcal{C}$ a closure operator.
If an immured element $K$ is removed from the covering $\mathcal{C}$, then $\mathcal{C}-\{K\}$ may not still be a covering which makes $XH$ and $VH$ be closure operators.
So we omit the discussion of the relationship between $\mathcal{I}_{XH}(\mathcal{C})$ and $\mathcal{I}_{XH}(\mathcal{C}-\{K\})$, and the relationship between $\mathcal{I}_{VH}(\mathcal{C})$ and $\mathcal{I}_{VH}(\mathcal{C}-\{K\})$.
\end{remark}

The following example illustrates the above remark.

\begin{example}
Suppose $K_{1}=\{1\}, K_{2}=\{1,2\}, K_{3}=\{2,3\}, K_{4}=\{3\}, K_{5}=\{1,2,3\}$, $\mathcal{C}_{1}=\{K_{1}, K_{2}, K_{3}, K_{4}\}$.
$N(1)=\{1\}, N(2)=\{2\}, N(3)=\{3\}$, thus $\{N(x):x\in E\}$ forms a partition of $E$.
Hence, $XH$ is a closure operator.
It is clear that $K_{1}$ is an immured element, and the neighborhoods induced by $\mathcal{C}-\{K_{1}\}$ are $N(1)=\{1,2\}, N(2)=\{2\}, N(3)=\{3\}$, thus $\{N(x):x\in E\}$ can not form a partition of $E$.
Hence, $\mathcal{C}-\{K_{1}\}$ is not a covering which makes $XH$ be a closure operator.
Let $\mathcal{C}_{2}=\{K_{1}, K_{2}, K_{3}, K_{5}\}$.
$N(1)=\{1\}, N(2)=\{2\}, N(3)=\{2,3\}, VH(\{1\})=\{1\}, VH(\{2\})= VH(\{3\})=\{2,3\}$, thus $\{VH(\{x\}:x\in E)\}$ forms a partition of $E$.
Hence, $\mathcal{C}_{2}$ is a covering such that $VH$ is a closure operator.
However, $\{VH(\{x\}):x\in E\}$ induced by $\mathcal{C}_{2}-\{K_{1}\}$ does not form a partition because $VH(\{1\})=\{1,2\}, VH(\{2\})=\{1,2,3\}, VH(\{3\})=\{2,3\}$.
It is clear that $K_{1}$ is an immured element in $\mathcal{C}_{2}$ and $\mathcal{C}_{2}-\{K_{1}\}$ is a covering which dose not make $VH$ be a closure operator.
\end{example}

\begin{theorem}
Let $\mathcal{C}$ be a covering on $E$ and $K$ be a reducible element. If $XH$ induced by $\mathcal{C}$ is a closure operator, then $XH$ induced by
$\mathcal{C}-\{K\}$ is also a closure operator.
Moreover, $\mathcal{I}_{XH}(\mathcal{C})=\mathcal{I}_{XH}(\mathcal{C}-\{K\})$ and $\mathcal{L}_{XH}(M(\mathcal{C}))=\mathcal{L}_{XH}(M(\mathcal{C}-\{K\}))$.
\end{theorem}

\begin{proof}
Since $XH$ induced by $\mathcal{C}$ is a closure operator, $\{N(x):x\in E\}$ induced by $\mathcal{C}$ forms a partition.
Based on the definition of $\mathcal{I}_{XH}(\mathcal{C})$ and Lemma 12, $XH$ induced by $\mathcal{C}-\{K\}$ is also a closure operator and $\mathcal{I}_{XH}(\mathcal{C})=\mathcal{I}_{XH}(\mathcal{C}-\{K\})$.
\end{proof}

\begin{corollary}
Let $\mathcal{C}$ be a covering on $E$. If $XH$ is a closure operator, then
$XH$ induced by $reduct(\mathcal{C})$ is a closure operator.
Moreover, $\mathcal{I}_{XH}(\mathcal{C})=\mathcal{I}_{XH}(reduct(\mathcal{C}))$ and $\mathcal{L}_{XH}(M(\mathcal{C}))=\mathcal{L}_{XH}(M(reduct(\mathcal{C})))$.
\end{corollary}

\begin{theorem}
Let $\mathcal{C}$ be a covering of $E$ and $K\in \mathcal{C}$ a reducible element. If $VH$ induced by $\mathcal{C}$ is a closure operator, then
$VH$ induced by $\mathcal{C}-\{K\}$ is also a closure operator.
Moreover, $\mathcal{I}_{VH}(\mathcal{C})=\mathcal{I}_{VH}(\mathcal{C}-\{K\})$ and $\mathcal{L}_{VH}(M(\mathcal{C}))=\mathcal{L}_{VH}(M(\mathcal{C}-\{K\}))$.
\end{theorem}

\begin{corollary}
Let $\mathcal{C}$ be a covering on $E$. If $VH$ induced by $\mathcal{C}$ is a closure operator, then $XH$ induced by $reduct(\mathcal{C})$ is also a closure operator.
Moreover, $\mathcal{I}_{VH}(\mathcal{C})=\mathcal{I}_{VH}(reduct(\mathcal{C}))$ and $\mathcal{L}_{VH}(M(\mathcal{C}))=\mathcal{L}_{VH}(M(reduct(\mathcal{C})))$.
\end{corollary}

\section{Conclusions}
\label{S:Conclusions}
This paper has studied the geometric lattice structures of covering based-rough sets through matroids.
The important contribution of this paper is that we have established a geometric lattice structure of covering-based rough sets through the transversal matroid induced by a covering and have presented three geometric lattice structures of covering-based rough sets through three types of approximation operators. Moreover, we have discussed the relationship among the four geometric lattice structures.
To study other properties of this type of geometric lattice structure and to study other geometric lattices from the viewpoint of other upper approximation operators is our future work.

\section{Acknowledgments}
This work is supported in part by the National Natural Science Foundation of China under Grant No. 61170128,  the Natural Science Foundation of Fujian Province, China, under Grant Nos. 2011J01374 and 2012J01294, and the Science and Technology Key Project of Fujian Province, China, under Grant No. 2012H0043.




\end{document}